\def\isarxiv{1} 

\ifdefined\isarxiv
\documentclass[11pt]{article}
\else
\documentclass{article}
\usepackage{iclr2021_conference}
\fi

\usepackage{comment}
\usepackage{amsmath}
\usepackage{amsthm}
\usepackage{amssymb}
\usepackage{float}
\usepackage{subfig}
\usepackage{wrapfig}
\usepackage{algorithm}
\usepackage{algpseudocode}
\usepackage{tikz}
\usepackage{comment}
\usepackage{hyperref}
\usepackage{multirow}
\hypersetup{colorlinks=true,citecolor=red,linkcolor=blue}
\usepackage{booktabs}
\usepackage{url}
\usepackage{bm}

\graphicspath{{./figs/}}

\ifdefined\isarxiv
\usepackage[margin=1in]{geometry}
\else

\fi

\definecolor{b2}{RGB}{51,153,255}
\definecolor{mygreen}{RGB}{80,180,0}
\definecolor{yl}{RGB}{255,80,0}
\definecolor{myl}{RGB}{180,80,20}
\definecolor{kl}{RGB}{180,0,20}

\newtheorem{theorem}{Theorem}[section]

\newtheorem{definition}[theorem]{Definition}

\newtheorem{corollary}[theorem]{Corollary}

\newtheorem{hypothesis}[theorem]{Hypothesis}

\renewcommand{\hat}{\widehat}

\renewcommand{\bar}{\overline}

\newcommand{\SAT}{{$\mathsf{3SAT}$}}
\newcommand{\ETH}{{$\mathsf{ETH}$}}

\newcommand{\MAXSAT}{{$\mathsf{MAX3SAT}$}}
\newcommand{\MAXSATB}{{$\mathsf{MAX3SAT}(B)$}}
\newcommand{\MAXESAT}{{$\mathsf{MAXE3SAT}$}}
\newcommand{\MAXESATB}{{$\mathsf{MAXE3SAT(B)}$}}
\newcommand{\NP}{{$\mathsf{NP}$}}
\newcommand{\RP}{{$\mathsf{RP}$}}
\newcommand{\CNF}{{$\mathsf{CNF}$}}
\newcommand{\mixcon}{{$\mathsf{MixCon}$}}
\newcommand{\unicon}{{$\mathsf{UniCon}$}}
\DeclareMathOperator{\poly}{poly}

\DeclareMathOperator{\R}{{\mathbb R}}

\newcommand{\dist}{\mathrm{dist}}%

\newcommand{\Binghui}[1]{{\color{myl}[Binghui: #1]}}
\newcommand{\Xiaoxiao}[1]{{\color{b2}[Xiaoxiao: #1]}}

\title{MixCon: Adjusting the Separability of Data Representations for Harder Data Recovery}

\date{}

\author{
Xiaoxiao Li\thanks{\texttt{xl32@princeton.edu}. Princeton University.}
\and
Yangsibo Huang\thanks{\texttt{yangsibo@princeton.edu}. Princeton University.}
\and
Binghui Peng\thanks{\texttt{bp2601@columbia.edu}. Columbia University.}
\and
Zhao Song\thanks{\texttt{zhaos@princeton.edu}. Columbia University and Princeton University.}
\and
Kai Li\thanks{\texttt{li@cs.princeton.edu}. Princeton University}
}

\begin{document}
\ifdefined\isarxiv

\begin{titlepage}
\maketitle
\begin{abstract}


To address the issue that deep neural networks (DNNs) are vulnerable to model inversion attacks, we design an objective function, which adjusts the separability of the hidden data representations, as a way to control the trade-off between data utility and vulnerability to inversion attacks. Our method is motivated by the theoretical insights of data separability in neural networking training and results on the hardness of model inversion. Empirically, by adjusting the separability of data representation, we show that there exist sweet-spots for data separability such 
that it is difficult to recover data during inference while maintaining data utility. 


 

%
\end{abstract}
\thispagestyle{empty}
\end{titlepage}

\else

\maketitle
\vspace{-15pt}
\begin{abstract}
\vspace{-10pt}

\end{abstract}
\vspace{-10pt}
\fi

\ifdefined\isarxiv\else\vspace{-2mm}\fi
\section{Introduction}
\ifdefined\isarxiv\else\vspace{-2mm}\fi
Over the past decade, deep neural networks have shown superior performances in various domains, such as visual recognition, natural language processing, robotics, and healthcare.
However, recent studies have demonstrated that machine learning models are vulnerable in terms of leaking private data~\cite{he2019model,zlh19, zhang2020secret}. 
Hence, preventing private data from being recovered by malicious attackers has become an important research direction in deep learning research. 


Distributed machine learning \cite{shokri2015privacy,kairouz2019advances} has emerged as an attractive setting to mitigate privacy leakage without requiring clients to share raw data. In the case of an edge-cloud distributed learning scenario, most layers  are commonly offloaded to the
cloud, while the edge device computes only a small number
of convolutional layers for feature extraction, due to power
and resource constraints \cite{kang2017neurosurgeon}. For example, service provider trains and splits a neural network at a ``cut layer,'' then deploys the rest of the layers to clients~\cite{vepakomma2018split}. Clients encode their dataset using those layers, then send the data representations back to cloud server using the rest of layers for inference \cite{teerapittayanon2017distributed,ko2018edge, vepakomma2018split}. This gives an untrusted cloud provider or a malicious participant a chance to steal sensitive inference data from the output of ``cut layer'' on the edge device side, i.e. inverting data from their outputs \cite{fredrikson2015model,zhang2020secret}.

In this paper, we investigate how to design a hard-to-invert \textit{data representation} function (or \textit{hidden data representation} function), which is defined as the output of the neural network's intermediate layer. We focus on defending data recovery during inference. The goal is to hide sensitive information and to protect data representations from being used to reconstruct the original data while ensuring that the resulted data representations are still informative enough for decision making. The core question here is how to achieve the goal.

We propose {\em data separability}, also known as the minimum (relative) distance between (the representation of) two data points, as a new criterion to investigate and understand the trade-off between data utility and hardness of data recovery. Recent theoretical studies show that if data points are separable in the hidden embedding space of a DNN model, it is helpful for the model to achieve good classification accuracy \cite{als19a}. However, larger separability is also easier to recover inputs. Conversely, if the embeddings are non-separable or sometimes overlap with one another, it is challenging to recover inputs. Nevertheless, the model may not be able to learn to achieve good performance.  
Two main questions arise. First, is there an effective way to adjust the separability of data representations?
Second, are there ``sweet spots'' that make the data representations difficult for inversion attacks while achieving good accuracy? 

This paper aims to answer these two questions by learning a feature extractor that can adjust the separability of data representations embedded by a few neural network layers. Specifically, we propose to add a self-supervised learning-based novel regularization term to the standard loss function during training. 
We conduct experiments on both synthetic and benchmark datasets to demonstrate that with specific parameters, such a learned neural network is indeed difficult to recover input data while maintaining data utility.

Our contributions can be summarized as: 
\begin{itemize}
    \item To the best of our knowledge, this is the first proposal to investigate the trade-off between data utility and data recoverability from the angle of data representation separability;
    \item
    We propose a simple yet effective loss term, Consistency Loss -- {\mixcon} for adjusting data separability;
    \item
    We provide the theoretical-guided insights of our method, including a new exponential lower bound on approximately solving the network inversion problem, based on the Exponential Time Hypothesis ({\ETH}); and
    \item
    We report experimental results comparing accuracy and data inversion results with/without incorporating {\mixcon}. We show {\mixcon} with suitable parameters makes data recovery difficult while preserving high data utility. 
\end{itemize}

The rest of the paper is organized as follow. We formalize our problem in Section~\ref{sec:pre}. In Section~\ref{sec:consistancy}, we present our theoretical insight and introduce the consistency loss. We demonstrate the experiment results in Section~\ref{sec:experiment}. 
We defer the technical proof and experiment details to Appendix.

\ifdefined\isarxiv\else\vspace{-2mm}\fi
\section{Preliminary}
\label{sec:pre}
\paragraph{Distributed learning framework.}
We consider a distributed learning framework, in which {\em users} and {\em servers} collaboratively perform inferences \cite{teerapittayanon2017distributed,ko2018edge,kang2017neurosurgeon}. We have the following assumptions: 1) Datasets are stored at the user sides. During inference, no raw data are ever shared among users and servers; 2) Users and servers use a split model~\cite{vepakomma2018split} where users encode their data using our proposed mechanism to extract data representations at a cut layer of a trained DNN. Servers take encoded data representations as inputs and compute outputs using the layers after the cut layer in the distributed learning setting; 3) DNN used in the distributed learning setting can be regularized by our loss function (defined later).


\paragraph{Threat model.} We consider the attack model with access to the shared hidden data representations during the client-cloud communication process. The attacker aims to recover user data  (i.e., pixel-wise recovery for images in vision task). To quantify the upper bound of privacy leakage under this threat model, we allow the attacker to have more power in our evaluation. In addition to having access to extracted features, we allow the attacker to see all network parameters of the trained model. 

\paragraph{Problem formulation.}
We focus on the pipelines combining local data representation learning with a global model learning manner at a high level. 
Formally, let $h:\R^{d} \rightarrow \R^{m}$ denote the local feature extractor function, which maps an input data $x \in \R^{d}$ to its feature representation $h(x) \in \R^{m}$. The local feature extractor is a shallow neural network in our setting.
The deep neural network on the server side is denoted as $g : \R^m \mapsto \R^C$, which performs classification tasks and maps the feature representation to one of $C$ target classes.
The overall neural network $f:  \R^d \mapsto \R^C $, and it can be written as $f = g \circ h$.

Our overall objectives are:
\ifdefined\isarxiv\else\vspace{-0.2cm}\fi
\begin{itemize}
\ifdefined\isarxiv\else\itemsep-0.1em\fi
\item Learn the feature representation mechanism (i.e. $h$ function) that safeguards information from unsolicited disclosure.
\item Jointly learn the classification function $g$, and the feature extraction function $h$ to ensure the information extracted is useful for high-performance downstream tasks.
\end{itemize}


\ifdefined\isarxiv\else\vspace{-2mm}\fi
\section{Consistency Loss for Adjusting Data Separability}
\label{sec:consistancy}
\ifdefined\isarxiv\else\vspace{-2mm}\fi
To address the issue of data recovery from hidden layer output, we propose a novel consistency loss in neural network training, as shown in Figure \ref{fig:theory}. Consistency loss is applied to the feature extractor $h$ to encourage encoding closed but separable representations for the data of different classes. Thus, the feature extractor $h$ can help protect original data from being inverted by an attacker during inference while achieving desirable accuracy. 
\begin{figure}[htbp]
    \centering
    \includegraphics[width=0.9\textwidth]{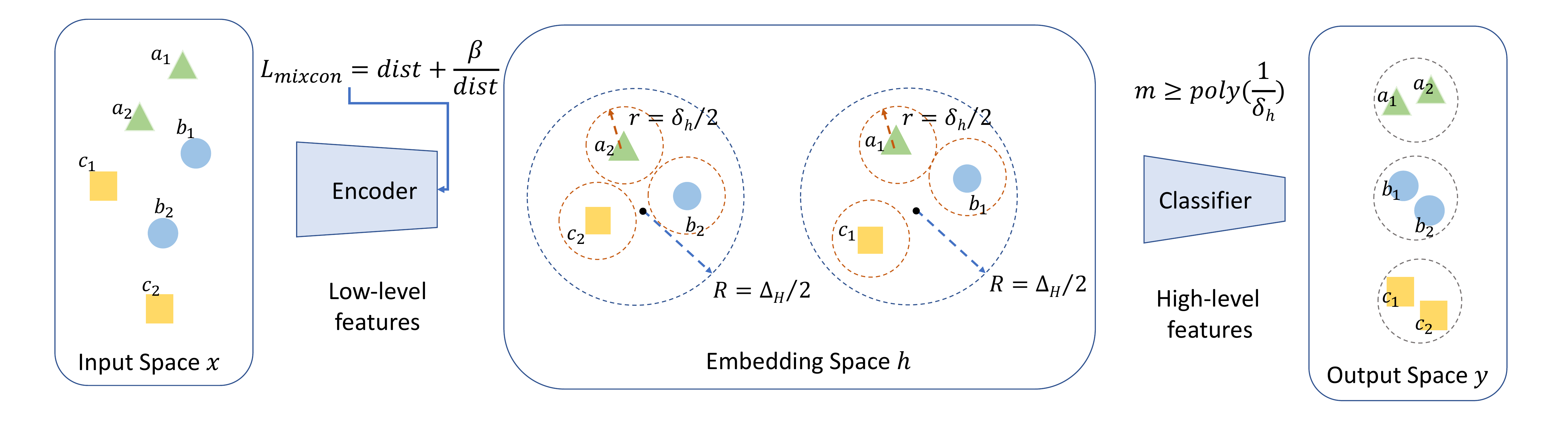}
    \caption{\small Schematic diagram of our data representation encoding scheme in deep learning pipeline. We show a simple toy example of classifying data points of triangles, squares, and circles. In embedding space (the middle block), data representations from different classes are constrained to a small ball with diameter $\Delta_H$, while they are separate from each other at least with distance $\delta_h$.  }
\label{fig:theory}
\end{figure}

\ifdefined\isarxiv\else\vspace{-2mm}\fi
\subsection{Data separation as a guiding tool}
\ifdefined\isarxiv\else\vspace{-2mm}\fi
Our intuition is to adjust the information in the data representations to a minimum such that downstream classification tasks can achieve good accuracy but not enough for data recovery through model inversion attacks \cite{he2019model}.
The question is, what is the right measure on the amount of information for successful classification and data security?
We propose to use data separability as the measure. This intuition is motivated by the theoretical results of deep learning. In particular, 
\begin{itemize}
    \ifdefined\isarxiv\else\vspace{-0.2cm}\fi
    \item Over-parameterized deep learning theory --- the well separated data requires a narrower network to train,
    \item In-approximation theory --- the worse separability of the data, the harder of the inversion problem.
\end{itemize}
\begin{definition}[Data separability]
Let $\delta_h$ denote the separability of hidden layer over all pairwise inputs $x_1, x_2, \cdots, x_n \in \R^d$, i.e.,
\begin{align*}
    \delta_h := \min_{i \neq j \in [n]} \| h(x_i) - h(x_j) \|_2 . ~~~\mathrm{(controlling~accuracy)} 
\end{align*}
Let $S$ denote a set of pairs that supposed to be close in hidden layer.
Let $\Delta_H$ denote the maximum distance with respect to that set $S$ 
\begin{align*}
    \Delta_H:= \max_{(i,j) \in S} \| h(x_i) - h(x_j) \|_2. ~~~\mathrm{(controlling~invertibility)}
\end{align*}
\end{definition}

\paragraph{Lower bound on data separability implies better accuracy}
Recent line of deep learning theory \cite{als19a} indicates that data separability is perhaps the only matter fact for learnability (at least for overparameterized neural network), leading into the following results.
\begin{theorem}
\label{th:sep}
Suppose the training data points are separable, i.e., $\delta_{h} > 0$.
If the width of a $L$-layer neural network with ReLU gates satisfies $m \geq \poly(n,d,L,1/\delta_h)$, initializing from a random weight matrix $W$, (stochastic) gradient descent algorithm can find the global minimum of neural network function $f$.
\end{theorem}
Essentially, the above theorem indicates that we can (provably) find global minimum of a neural network given well separated data, and better separable data points requires narrower neural network and less running time.

\paragraph{Upper bound on data separability implies hardness of inversion.}
When all data representation is close to each other, i.e. $\Delta_{H}$ is sufficiently small, we expect the inversion problem is hard. We support this intuition by proving that the neural network inversion problem is hard to {\em approximate} within some constant factor when assuming {\NP}$\neq${\RP}.\footnote{The class {\RP} consists of all languages $L$ that have a polynomial-time randomized algorithm $A$ with the following behavior: If $x\notin L$, then A always rejects $x$ (with probability 1). If $x\in L$, then A accepts $x$ in L with probability at least $1/2$.}

Existing work~\cite{ljdd19} indicates that the decision version of the neural network inversion problem is {\NP}-hard. However, this is insufficient since it is usually easy to find an approximate solution, which could leak much information on the original data. It is an open question whether the approximation version is also challenging.
We strengthen the hardness result and show that by assuming {\NP}$\neq${\RP}, it is hard to recovery an input that approximates the hidden layer representation. 
Our hardness result implies that given hidden representations are close to each other, no polynomial time can distinguish their input. Therefore, it is impossible to recover the real input data in polynomial time.

\begin{theorem}[Informal]
\label{thm:rp}
Assume {\NP}$\neq${\RP}, there is no polynomial time algorithm that is able to give a constant approximation to thee neural network inversion problem.
\end{theorem}

The above result only rules out the polynomial running time recovery algorithm but leaves out the possibility of a subexponential time algorithm.
To further strengthen the result, we assume the well-known Exponential Time Hypothesis ({\ETH}), which is widely accepted in the computation complexity community.

\begin{hypothesis}[Exponential Time Hypothesis ({\ETH}) \cite{ipz98}]
There is a $\delta>0$ such that the {\SAT} problem cannot be solved in $O(2^{\delta n})$ time.
\end{hypothesis}

Assuming {\ETH}, we derive an exponential lower bound on approximately recovering the input.
\begin{corollary}[Informal]
\label{cor:eth}
Assume $\mathsf{ETH}$, there is no $2^{o(n^{1 - o(1)})}$ time algorithm that is able to give a constant approximation to neural network inversion problem.
\end{corollary}

\ifdefined\isarxiv\else\vspace{-2mm}\fi
\subsection{Consistency loss --- {\mixcon}}
\ifdefined\isarxiv\else\vspace{-2mm}\fi
Follow the above intuitions, we propose a novel loss term {\mixcon} loss --- ${\cal L}_{\mathrm{mixcon}}$ --- to incorporate in training. {\mixcon} adjusts data separability by forcing the consistency of hidden data representations from different classes. 
This additional loss term balances the data separability, punishing feature representations that are too far or too close to each other. Noting that we choose to mix data from different classes instead of the data within a class, in order to bring more confusion in embedding space and potentially hiding data label information\footnote{We show the comparison in Appendix \ref{sec:unicon}.}. 

\paragraph{{\mixcon} loss ${\cal L}_{\mathrm{mixcon}}$:}  We add consistency penalties to force the data representation of $i$-th data in different classes to be similar, while without any overlapping for any two data points. 
    \begin{align}
    \label{eq:mixcon}
         {\cal L} _{ \mathrm{mixcon} } :=  ~ \frac{1}{ p } \frac{1}{ |{\cal C}| \cdot (|{\cal C}| - 1) } \sum_{i=1}^{ p }  \sum_{c_1 \in {\cal C}} \sum_{c_2 \in {\cal C}} ( \dist(i,c_1,c_2) + \beta/\dist(i,c_1,c_2) ).
    \end{align}
A practical choice for the pairwise distance is $ \dist(i,c_1,c_2) = \| h( x_{i,c_1} ) - h( x_{i,c_2} ) \|_2^2 $ \footnote{In practice, we normalize $\|h(x)\|_2$ to 1. To avoid division by zero, we can use a positive small $\epsilon~(\ll 1)$ and threshold distance to the range of $[\epsilon,~1/\epsilon]$.}, where $x_{i,c}$ is the $i$-th input data point in class $c$, $p:=\min_{c\in {\cal C}} |c|$, and $\beta>0$ balances the data separability. The first term punishes large distance while the second term enforces sufficient data separability. In general, we could replace $(\dist + \beta/\dist)$ by convex functions with asymptote shape on non-negative domain, that is, function with value reaches infinity on both ends of $[0,\infty)$.

We consider the classification loss 
\begin{align}
\label{eq:celoss}
{\cal L}_{\mathrm{class}} := -\sum_{i=1}^N \sum_{c=1}^{C} y_{i,c} \cdot \log(\hat{y}_{i,c}) ~~~\text{(cross~entropy)}
\end{align}  
where $y_{i} \in \R^C$ is the one-hot representation of true label and $\hat{y}_{i}=f(x_i) \in \R^C$ is the prediction score of data $i \in \{1,\dots,N\}$. The final objective function is ${\cal L} := {\cal L}_{\mathrm{class}} + \lambda \cdot {\cal L}_{\mathrm{mixcon}} $. We simultaneously train $h$ and $g$, where $\lambda$ and $\beta$ are tunable hyper-parameters associated with consistency loss regularization to adjust separability. We discuss the effect of $\lambda$ and $\beta$  in experiments (Section \ref{sec:experiment}).

\section{Experimental Results}
\label{sec:experiment}


\subsection{Data recovery model} To empirically evaluate the quality of inversion, we formally define the white-box data recovery (inversion) model \cite{he2019model} used in our experiments. The model aims to solve an optimization problem in the input space.  Given a representation $z=h(x)$, and a public function $h$ (the trained network that generates data representations), the inversion model tries to find the original input $x$:
\begin{equation} \label{eq:attack-exp}
      x^* = \arg\min_{s} {\cal L} ( h(s) , z) + \alpha \cdot {\cal R}(s)  
\end{equation}
where ${\cal L}$ is the loss function that measures the similarity between $h(s)$ and $z$, and ${\cal R}$ is the regularization term. We specify ${\cal L}$ and ${\cal R}$ used in each experiment later. 
We solve Eq. \eqref{eq:attack-exp} by iterative gradient descent.

\subsection{Experiments with synthetic data} In this section, we want to answer the following questions:
\begin{enumerate}
    \item[Q1] What is the impact of having $\beta$ in Eq.\eqref{eq:mixcon} to bound the smallest data pairwise distance?
    \item[Q2] Is feature encoded with {\mixcon} mechanism harder to invert?  
\end{enumerate}

To allow precise manipulation and straightforward visualization for data separability, our experiments use generated synthetic data with a 4-layer fully-connected network.


\paragraph{Network, data generation and training.} 
We defined the network as
\begin{align*}
y = q ( \text{softmax} ( f(x) ) ), ~~~ f(x) = W_4 \cdot \sigma(W_3 \cdot \sigma(W_2 \cdot (\sigma(W_1x+b_1))+b_2)+b_3)+b_4
\end{align*}
$x \in \R^{10},~W_1 \in \R^{500 \times 10},~W_2 \in \R^{2 \times 500}, ~W_3 \in \R^{100 \times 2}, ~W_4 \in \R^{2 \times 100}, ~b_1 \in \R^{500},~b_2\in\R^{2},~b_3\in\R^{100},~b_4\in\R^{2}$. For a vector $z$, we use $q(z)$ to denote the index $i$ such that $|z_i| > |z_j|$, $\forall j \neq i$.  We initialize each entry of $W_k$ and $b_k$ from ${\cal N}(u_k,1)$, where $u_k~\sim{\cal N}(0,\alpha)$ and $k \in \{1,2,3,4\}$. 

We generate synthetic samples $(x,y)$ from two multivariate normal distribution. Positive data are sampled from ${\cal N}(0,I)$, and negative data are sampled from ${\cal N}(-1,I)$, ending up with 800 training samples and 200 testing samples, where the covariance matrix $I$ is an identity diagonal matrix. ${\cal L}_{\mathrm{mixcon}}$ is applied to the 2nd fully-connected layer. 

We train the network for 20 epochs with cross-entropy loss and SGD optimizer with $0.1$ learning rate. We apply noise to the labels by randomly flipping $5\%$ of labels to increase training difficulty.



\paragraph{Testing setup.}
We compare the results under the following settings: 
\begin{itemize}
    \item Vanilla: training using only ${\cal L}_{\mathrm{class}}$.
    \item {\mixcon}: training with {\mixcon} loss with parameters ($\lambda,~\beta$)\footnote{$\lambda$ is the coefficient of penalty and  $~\beta$ is balancing term for data separability.}.
\end{itemize} 
We perform model inversion using Eq.~\eqref{eq:attack-exp} without any regularization term ${\cal R}(x)$ and ${\cal L}$ is the $\ell_1$-loss function. Detailed optimization process is listed in Appendix~\ref{app:invert1}. 


\begin{figure*}[!t]
    \captionsetup[subfigure]{justification=centering}
    \centering
    \subfloat[\small Vanilla $E=0$]{\includegraphics[width=0.165\linewidth]{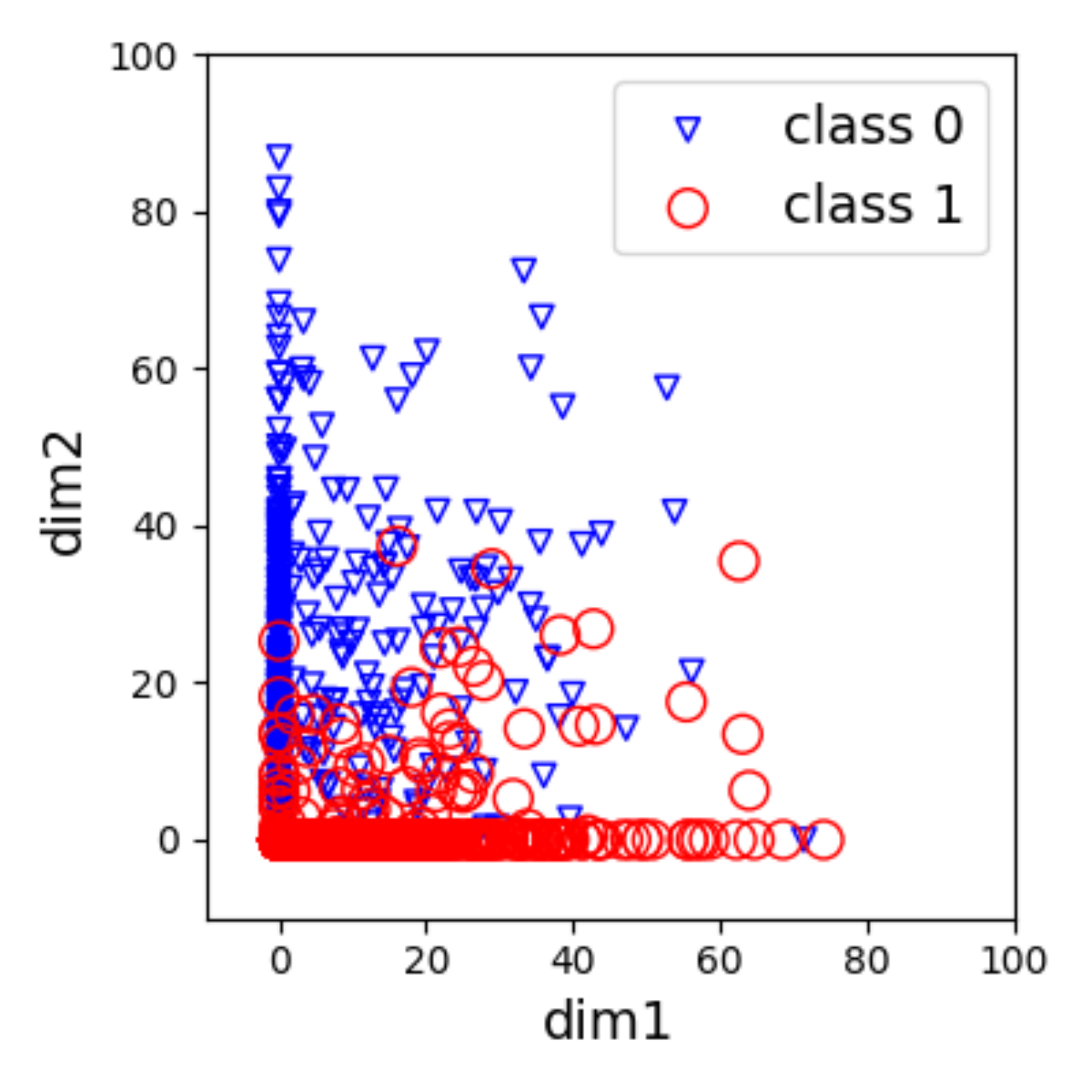}}
    \subfloat[\small Vanilla $E=20$]{\includegraphics[width=0.16\linewidth]{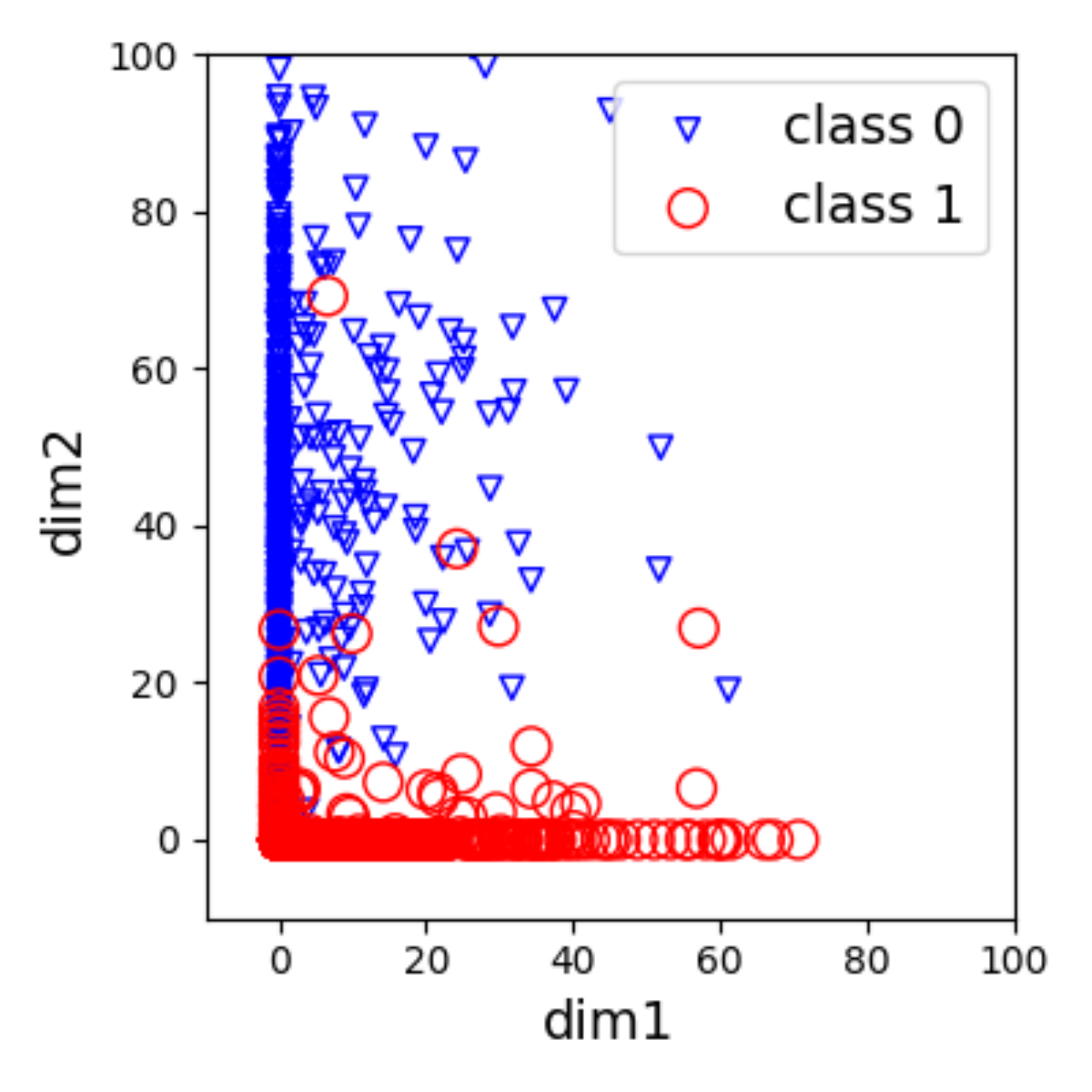}} 
    \subfloat[\small {\mixcon} $\beta=0.01$ $E=0$ ]{\includegraphics[width=0.165\linewidth]{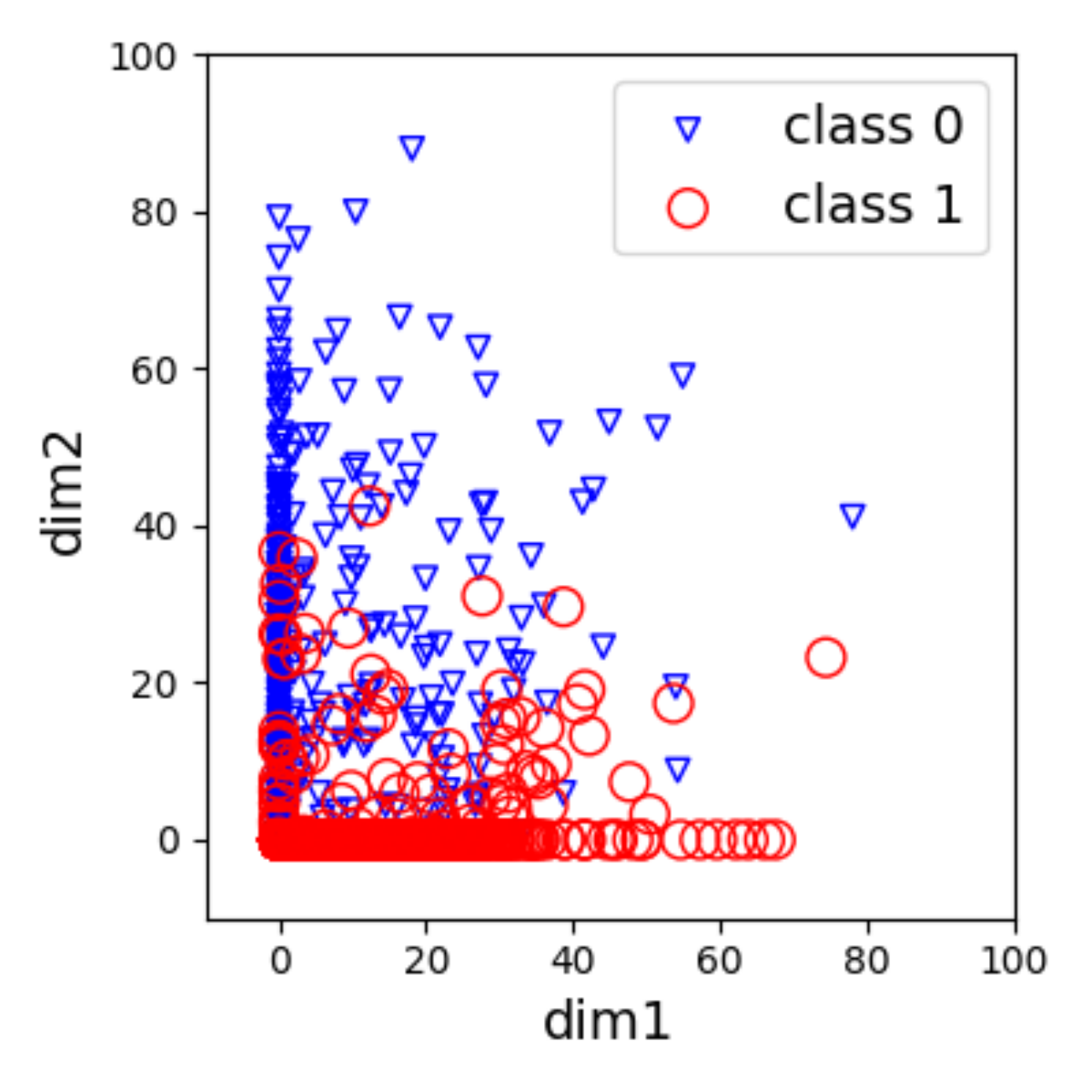}}
    \subfloat[\small {\mixcon} $\beta=0.01~E=20$]{\includegraphics[width=0.165\linewidth]{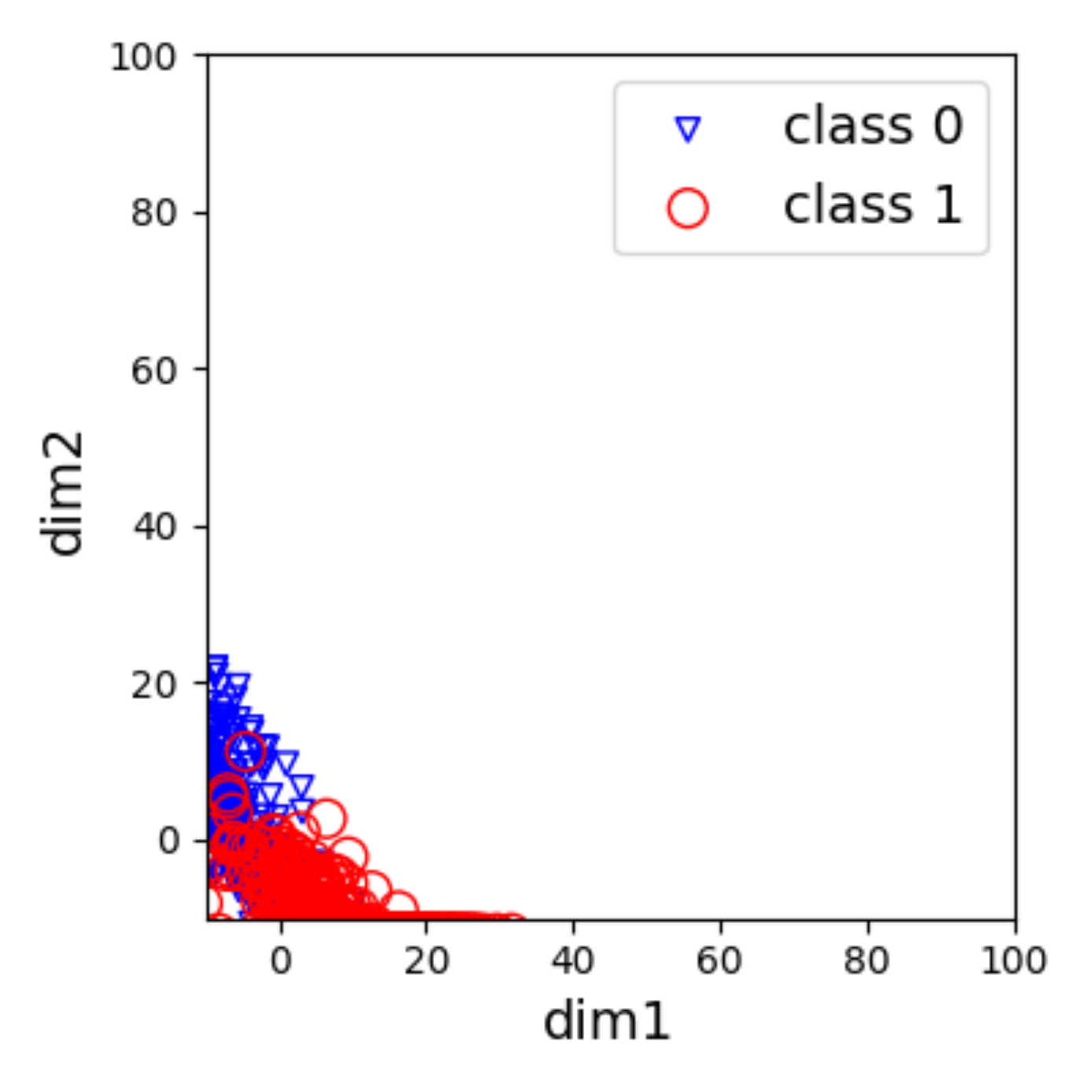}} 
    \subfloat[\small {\mixcon} $\beta=0$  $E=0$]{\includegraphics[width=0.165\linewidth]{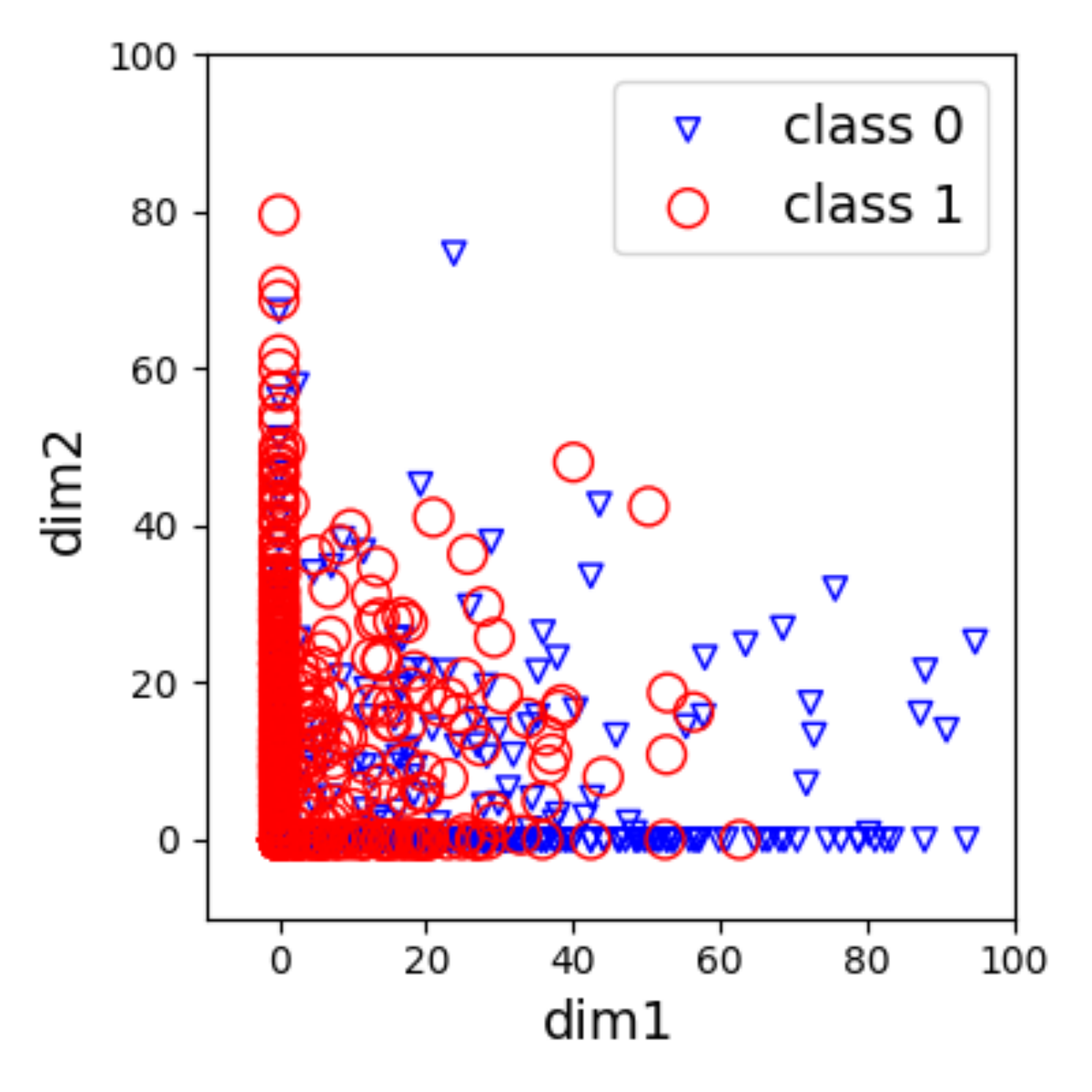}}
    \subfloat[\small {\mixcon} $\beta=0$ $E=20$]{\includegraphics[width=0.165\linewidth]{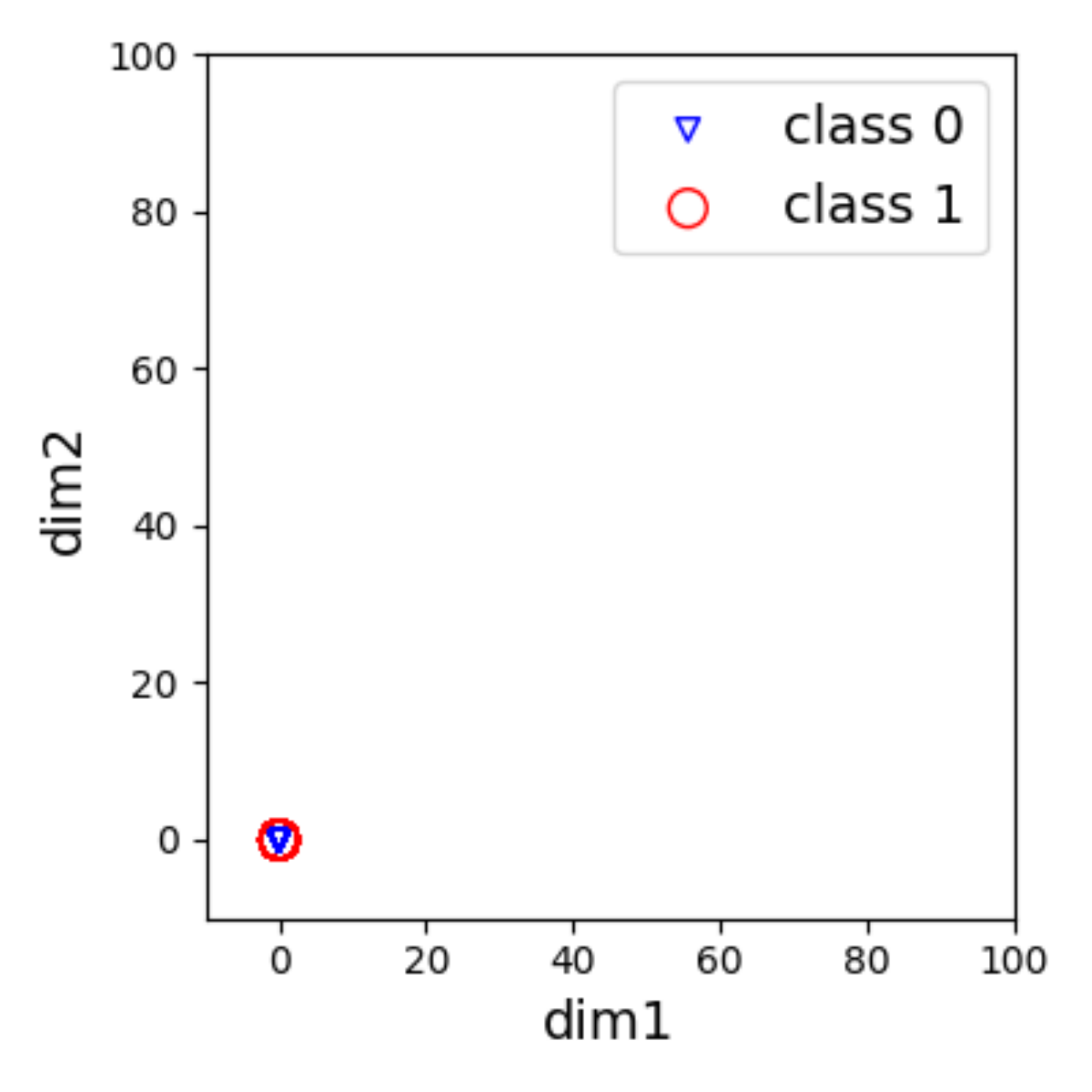}}
    \caption{\small Data hidden representation $h(x) \in \R^2$ from the 2nd fully-connected layer of synthetic data at different epoch ($E$).  Two settings of {\mixcon} are given default $\lambda=0.1$ but have different $\beta$. 
    Compare to Vanilla, {\mixcon} squeezes data representations to a smaller space over training. When $\beta=0$, {\mixcon} map all data to $h(x)=(0,0)$, which is not learnable.}
    \label{fig:syn_2d}
\end{figure*}


\paragraph{Results.} To answer Q1, we visualize the change of data representations at initial and ending epochs in Figure \ref{fig:syn_2d}. First, in vanilla training (Figure \ref{fig:syn_2d} a-b), data are dispersively distributed and enlarge their distance after training. The obvious difference for {\mixcon} training (Figure \ref{fig:syn_2d} c-f) is that data representations become more and more gathering through training. Second, we direct the data utility results of Vanilla and two ``default'' {\mixcon} settings -- ($\lambda = 0.1, \beta=0.01$) and ($\lambda = 0.1, \beta=0$) to Table~\ref{tab:table1}. When $\beta=0$, {\mixcon} achieves chance accuracy only as it encodes all the $h(x)$ to hidden space (0,0) (Figure \ref{fig:syn_2d} f). While having $\beta>0$ balancing the separability, {\mixcon} achieves similar accuracy as Vanilla.

\begin{table}[t]
\centering
\resizebox{0.9\textwidth}{!}{%
\begin{tabular}{l|c|ccc|ccc}
\hline
 & \multirow{2}{*}{Vanilla} & \multicolumn{3}{c|}{{\mixcon} $\beta = 0.01$} & \multicolumn{3}{c}{{\mixcon} $\beta = 0$} \\
 &  & default & deeper net & wider net & default & deeper net & wider net \\ \hline
Train Accuracy ($\%$) & 91.5 & 88.9 & 89.5 & 91.5 & 50.0 & 50.0 & 50.0 \\
Test Accuracy ($\%$) & 91.5 & 88.0 & 88.5 & 90.5 & 50.0  & 50.0 & 50.0 \\ \hline
\end{tabular}%
}
\caption{\small Data utility (accuracy). Vanilla is equivalent to ($\lambda = 0 ,~\beta = 0 $). Two {\mixcon} ``default'' settings both use $\lambda =0.1 $ but vary in $\beta=0.01$ and $\beta=0$. ``Deeper''/``Wider'' indicate increasing the depth / width of layers in the network on server side $g(x)$.}
\label{tab:table1}
\end{table}

Based on Theorem \ref{th:sep}, we further present two strategies to ensure reasonable accuracy while comprise of reducing data separability by increasing the depth or the width of the layers $g(z)$, the network after the layer that is applied ${\cal L}_{\mathrm{mixcon}}$. 
In practice, we add two more fully-connected layers with 100 neurons after the 3nd layer for ``deeper'' $g(x)$, and change the number of neurons on the 3nd layer to 2048 for ``wider'' $g(x)$. We show the utility results in Table \ref{tab:table1}. Using deeper or wider $g(z)$, {\mixcon} ($\lambda = 0.1, \beta=0.01$)  improves accuracy. Whereas {\mixcon} ($\lambda = 0.1, \beta=0$) fails, because zero data separability is not learnable no matter how $g(z)$ changes. This gives conformable answer that $\beta$ is an important factor to guarantee neural network to be trainable. 

\begin{wraptable}{r}{7cm}
\centering
\resizebox{0.33\textwidth}{!}{%
\begin{tabular}{l|ccc}
\hline
 & \multirow{2}{*}{Vanilla} & {\mixcon} & {\mixcon} \\
 &  & (0.1, 0.01) & (0.1, 0) \\ \hline
MSE &1.92 & 2.08  & 2.35 \\
MCS & 0.169 & 0.118 & 0.161\\ \hline
\end{tabular}%
}
\caption{\small Inversion results on synthetic dataset. Higher MSE or lower MCS indicates a worse inversion. $(\lambda, \beta)$ denoted in header.}\label{tab:table2}
\end{wraptable}
To answer Q2, we evaluate the quality of data recovery using the inversion model. We use both mean-square error (MSE) and mean-cosine similarity (MCS) of $x$ and $x^*$ to evaluate the data recovery accuracy. We show the quantitative inversion results in Table \ref{tab:table2}. Higher MSE or lower MCS indicates a worse inversion.
Apparently, data representation from {\mixcon} trained network is more difficult to recover compared to Vanilla strategy.

\ifdefined\isarxiv\else\vspace{-4mm}\fi
\subsection{Experiments with benchmark datasets}
\ifdefined\isarxiv\else\vspace{-2mm}\fi
In this section, we would like to answer the following questions:
\begin{itemize}
    \item[Q3] How does {\mixcon} loss affect data separability and accuracy on image datasets?
    \item[Q4] Are there parameters ($\lambda, \beta$) in {\mixcon} (Eq.~\eqref{eq:mixcon}) to reach a ``sweet-spot'' for data utility and the quality of defending data recovery?
\end{itemize}
\begin{figure*}[!t]
    \centering
    \subfloat[MNIST]{\includegraphics[width=0.8\linewidth]{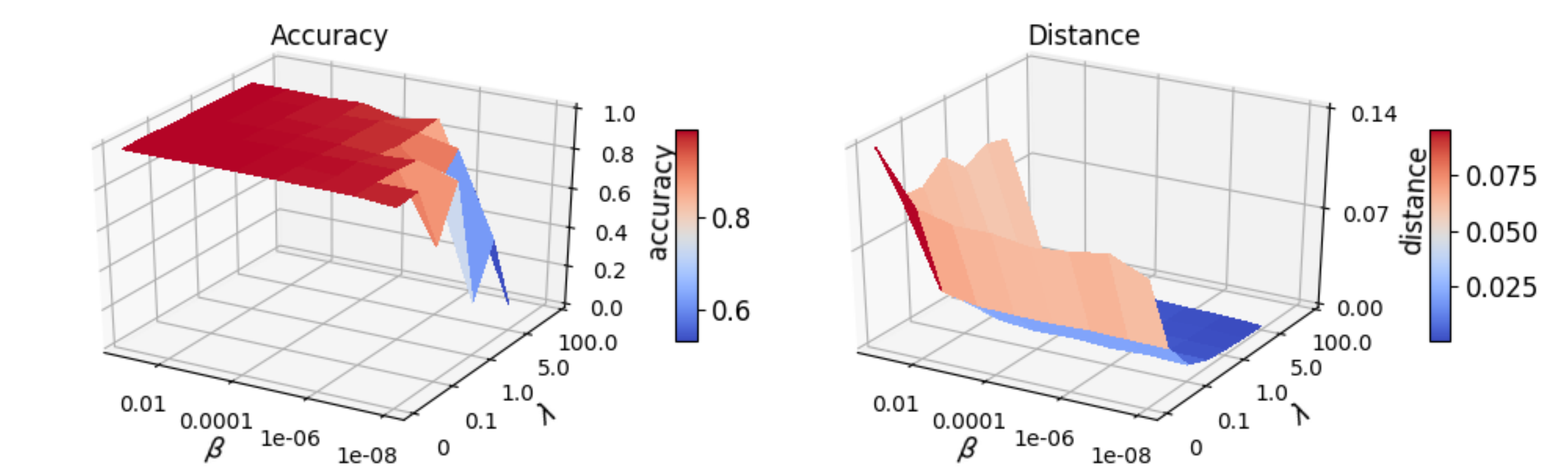}}\hspace{2mm}
    \subfloat[FashionMNIST]{\includegraphics[width=0.8\linewidth]{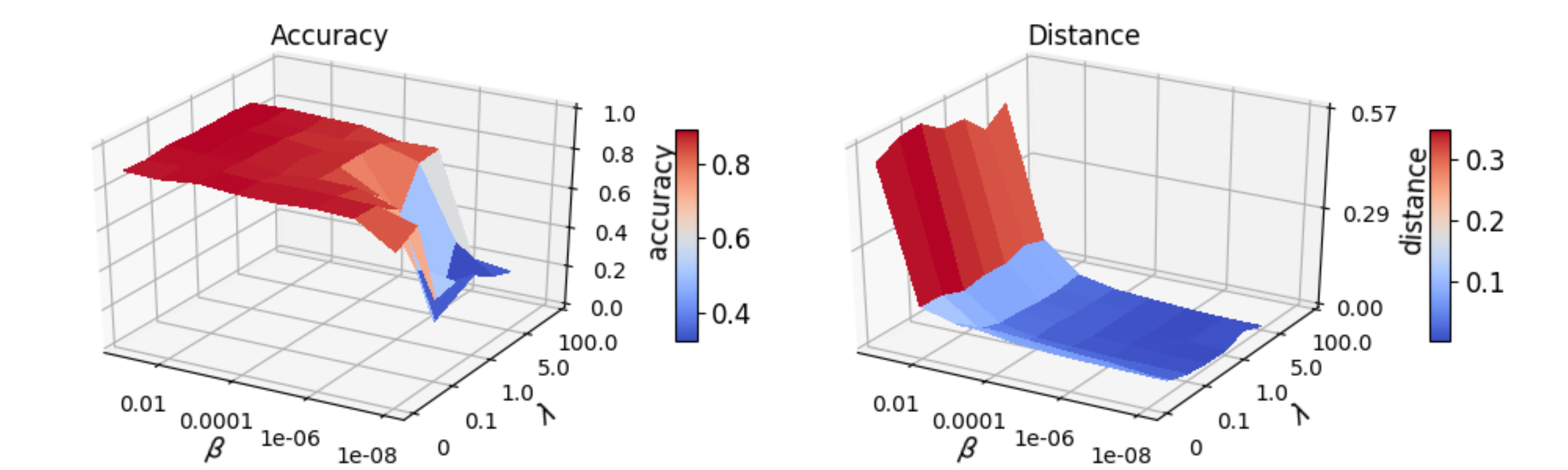}}\hspace{2mm}
    \subfloat[SVHN]{\includegraphics[width=0.8\linewidth]{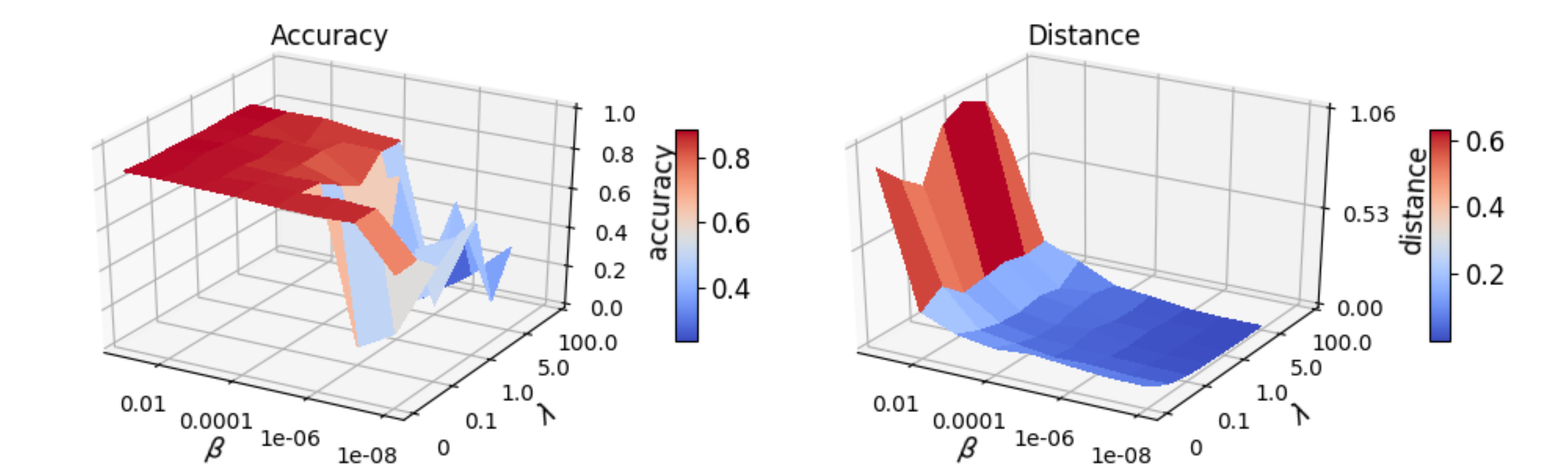}}\hspace{2mm}
    \caption{\small The trade-off between data separability and data utility. We show testing accuracy and mean pairwise distance (data separability) on three datasets with different $\lambda$ and $\beta$.  $\lambda$ and $\beta$ show complementary effort on adjusting data separability. A sweet-spot can be found at the ($\lambda$, $\beta$), achieving small data separability and high data utility.}
    \label{fig:digit_acc}
\end{figure*}

\begin{figure}[t]
    \centering
    \includegraphics[width=\linewidth]{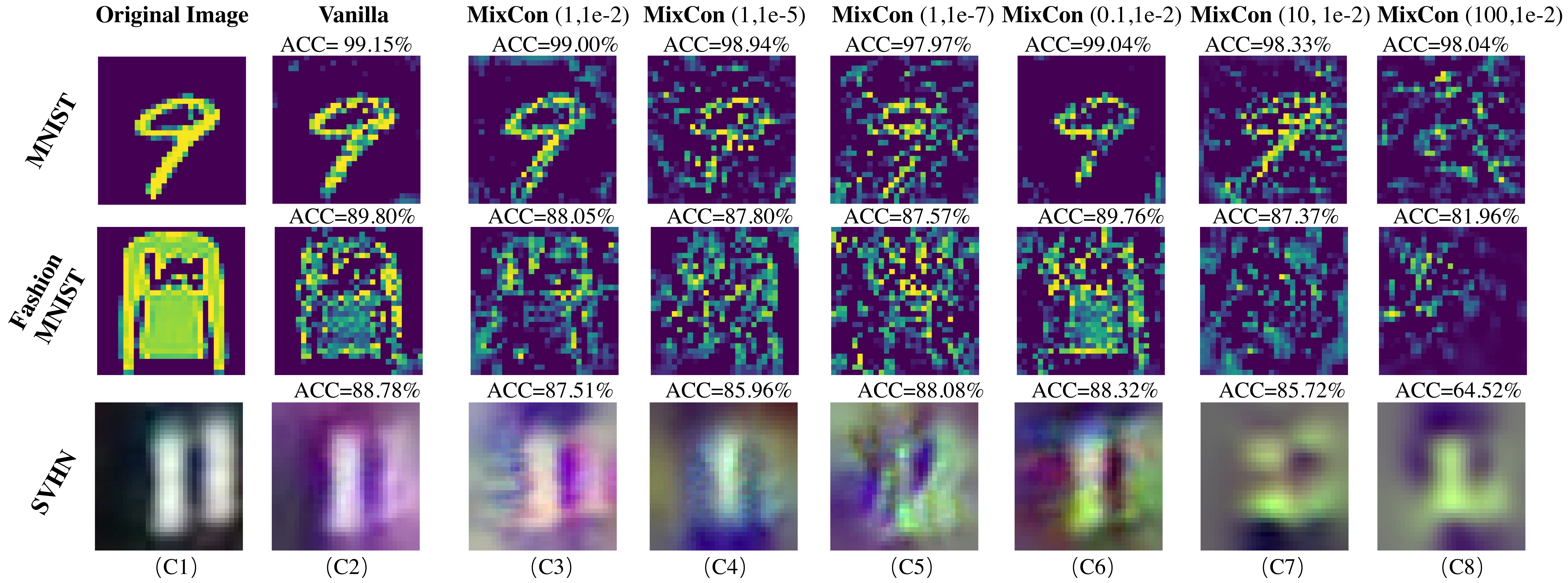}
    \caption{\small Qualitative evaluation for image inversion results. $(\lambda,~\beta)$ settings of {\mixcon} denoted on the header. The corresponding testing accuracy of each dataset is denoted on the top of each row. Compared to vanilla training, inversions from the {\mixcon} model are less realistic and distinguishable from the original images without significant accuracy dropping.}
    \label{fig:inversion}
\end{figure}

\paragraph{Network, datasets and training setup.}
The neural network architecture used in the experiments is LeNet5~\cite{lecun2015lenet}\footnote{We change input channel to 3 for SVHN dataset.}.  ${\cal L}_{\mathrm{mixcon}}$ is applied to the outputs of the 2nd convolutional layer blocks of LeNet5. The experiments use three datasets:
MNIST \cite{lecun1998mnist}, Fashion-MNIST \cite{xiao2017fashion}, and SVHN \cite{netzer2011svhn}.

Neural network is optimized using cross-entropy loss and SGD optimizer with learning rate 0.01 for 20 epochs. We do not use any data augmentation or manual learning rate decay. {\mixcon} loss is applied to the output of 2nd convolutional layer blocks in LeNet5. We train the model with different pairs of $(\lambda,\beta)$ in Eq.~\eqref{eq:mixcon} for the following testing. Specifically, we vary $\lambda$ from: $\{0.01, 0.1, 0.5, 1, 2, 5, 10, 100\}$ and  $\beta$  from: $\{10^{-2}, 10^{-3}, 10^{-4}, 10^{-5}, 10^{-6}, 10^{-7}, 10^{-8} \}$.

\paragraph{Testing setup.}  
We record the testing accuracy and pairwise distance of data representation under each pair of $(\lambda,\beta)$ for each dataset. Following a recent model inversion method \cite{he2019model},  we define ${\cal L}$ in Eq.~\eqref{eq:attack-exp} as $\ell_2$-loss function, ${\cal R}$ as the regularization term capturing the total variation of a 2D signal defined as $\mathcal{R}(a) = \sum_{i,j} ( (a_{i+1, j} - a_{i,j})^2 + (a_{i, j+1} - a_{i,j})^2 )^{1/2}$. The inversion attack is applied to the output of 2nd convolutional layer blocks in LeNet5 and find the optimal of Eq.~\eqref{eq:attack-exp} though SGD optimizer. Detailed optimization process is listed in Appendix~\ref{app:invert2}. 

We use metrics normalized structural similarity index metric (SSIM)~\cite{wbs+04} and perceptual similarity (PSIM)~\cite{johnson2016perceptual} to measure the similarity between the recovered image and the original image. The concrete definitions of SSIM and PSIM are listed in Appendix \ref{sec:metric}.
\ifdefined\isarxiv \else\vspace{-4mm}\fi
\paragraph{Results}
To answer Q3, we plot the complementary effects of $\lambda$ and $\beta$ in Figure \ref{fig:digit_acc}.  
 Note that $\beta$ bounds the minimal pairwise of data representations, and $\lambda$ indicate the penalty power on data separability given by {\mixcon}.
Namely, a larger $\lambda$ brings stronger penalty of {\mixcon}, which enhances the regularization of data separability and results in lower accuracy. Meanwhile, with a small $\beta$, $\lambda$ is not necessary to be very large, as smaller $\beta$ leads to a smaller bound of data separability, thus resulting in lower accuracy. Hence, $\lambda$ and $\beta$ work together to adjust the separability of hidden data representations, which can affect on data utility. 

To answer Q4, we evaluate the quality of inversion qualitatively and quantitatively through a model inversion attack defined in ``Test setup'' paragraph. Specifically, for each private input $x$, we execute the inversion attack on $h_\text{mixcon}(x)$ and $h_\text{vanilla}(x)$ of testing images. As it is qualitatively shown in Figure \ref{fig:inversion}, first, the recovered images using model inversion from {\mixcon} training (such as given $(\lambda, \beta)$ $\in$ {$\{ (1,  1\times10^{-7})$, $(10, 1\times10^{-2})$, $(100, 1\times10^{-2}) \}$}) are visually different from the original inputs, while the recovered images from Vanilla training still look similar to the originals. Second, with the same $\lambda$ (Figure \ref{fig:inversion} column c3-c5), the smaller the $\beta$ it is, the less similar of the recovered images to original images. Last, with the same $\beta$ (Figure \ref{fig:inversion} column c3 and c6-c8), the larger the $\lambda$ it is, the less similar of the recovered images to original images. 

\begin{table}[t]
\centering
\resizebox{\textwidth}{!}{%
\begin{tabular}{l|cc|cc|cc} \toprule
 & \multicolumn{2}{c|}{{\bf MNIST}} & \multicolumn{2}{c|}{{\bf FashionMNIST}} & \multicolumn{2}{c}{{\bf SVHN}} \\
 & \multicolumn{1}{c}{\begin{tabular}[c]{@{}c@{}}Vanilla\\ - \end{tabular}} & \multicolumn{1}{c|}{\begin{tabular}[c]{@{}c@{}}{\mixcon}\\ $(\lambda=1.0, \beta = 10^{-4})$\end{tabular}} & \multicolumn{1}{c}{\begin{tabular}[c]{@{}c@{}}Vanilla\\ - \end{tabular}} & \multicolumn{1}{c|}{\begin{tabular}[c]{@{}c@{}}{\mixcon}\\ $(\lambda=1.0, \beta = 10^{-4})$\end{tabular}} & \multicolumn{1}{c}{\begin{tabular}[c]{@{}c@{}}Vanilla\\ - \end{tabular}} & \multicolumn{1}{c}{\begin{tabular}[c]{@{}c@{}} {\mixcon}\\ $(\lambda=0.5, \beta = 10^{-4})$\end{tabular}}  \\
 \hline
Acc ($\%$) & $99.1$ & $98.6$ & $89.8$ & $88.9$ & $88.4$ & $88.2$  \\
SSIM & $0.64 \pm 0.11(0.83)$  & $ 0.14 \pm 0.11 (0.48)$ & $0.43 \pm 0.17 (0.78)$ & $0.17 \pm 0.09 (0.52)$ & $0.76 \pm 0.19 (0.92)$ & $0.61 \pm 0.15 (0.84)$ \\
PSIM & $0.78 \pm 0.05 (0.88)$ & $0.44 \pm 0.07 (0.69)$ & $0.71 \pm 0.13 (0.92)$ & $0.42 \pm 0.08 (0.66)$  & $0.69 \pm 0.07 (0.81)$ & $0.59 \pm 0.07 (0.72)$\\
\bottomrule
\end{tabular}%
}
\caption{\small Quantitative evaluations for image recovery results. For fair evaluation, we match the data utility (accuracy) for Vanilla and {\mixcon}. SSIM and PSIM are measured on 100 testing samples. Those scores are presented in mean $\pm$ std and worst-case (in parentheses) format. Lower scores indicate harder to invert.} 
\ifdefined\isarxiv\else\vspace{-4mm}\fi
\label{tab:similarity}
\end{table}
Further, we quantitatively measure the inversion performance by reporting the averaged similarity between 100 pairs of recovered images by the inversion model and their original samples. We select $(\lambda,~\beta)$ to match the accuracy results of {\mixcon} to be as good as Vanilla training (see Accuracy in Table~\ref{tab:similarity}), and investigate if {\mixcon} makes the inversion attack harder. The inverted results (see SSIM and PSIM in Table \ref{tab:similarity}) are reported in the format of mean $\pm$ std and the worst case (the best-recovered data) similarity in parentheses for each metric. Both qualitative and quantitative results agree with our hypothesis that 1) adding ${\cal L}_{\mathrm{mixcon}}$ in network training can reduce the mean pairwise distance (separability) of data hidden representations; and 2) maller separability make it more difficult to invert original inputs. Thus by visiting through possible $(\lambda,~\beta)$, we are able to find a spot, where data utility is reasonable but harder for data recovery, such as $(\lambda =100,~\beta=1e-2)$ for MNIST (Figure \ref{fig:inversion}).

\section{Discussion and Conclusion}
In this paper, we have proposed and studied the trade-off between  data  utility  and  data  recovery from the angle of the separability of hidden data representations in deep neural network. We propose using {\mixcon}, a consistency loss term, as an effective way to adjust the data separability.  Our proposal is inspired by theoretical data separability results and a new exponential lower bound on approximately solving the network inversion problem, based on the Exponential Time Hypothesis ({\ETH}).

We conduct two sets of experiments, using synthetic and benchmark datasets, to show the effect of adjusting data separability on accuracy and data recovery.  Our theoretical insights help explain our key experimental findings: {\mixcon} can effectively adjust the separability of hidden data representations, and one can find ``sweet-spot'' parameters for {\mixcon} to make it difficult to recover data while maintaining data utility. Our experiments are limited to small benchmark datasets in the domain of image classifications.  It will be helpful to conduct experiments using large datasets in multiple domains to further the study of the potential of adjusting data separability of data representations to trade-off between data utility and data recovery.

\ifdefined\isarxiv
\section*{Acknowledgments}
This project is supported in part by Princeton University fellowship, Schmidt Foundation, Simons Foundation, NSF, DARPA/SRC, Google and Amazon AWS. 
\else
\fi



\ifdefined\isarxiv
\addcontentsline{toc}{section}{References}
\bibliographystyle{alpha}
\else
\bibliographystyle{iclr2021_conference}
\fi
\bibliography{ref}

\newpage
\appendix 

\paragraph{Roadmap of Appendix} The Appendix is organized as follows.
We discuss related work in Section~\ref{sec:relate}.
We provide theoretical analysis in Section~\ref{sec:theory}.
The details of data recovery experiment are in Section~\ref{sec:recovery} and additional experiment details are in Section~\ref{sec:add-exp}.

\section{Related Work}\label{sec:relate}

\subsection{Hardness and Neural Networks}
When there are no further assumptions, neural networks have been shown hard in several different perspectives. \cite{br92} first proved that learning the neural network is NP-complete. Different variant hardness results have been developed over past decades \cite{ks09,dan16,ds16,gkkt17,lss14,wzc+18,mr18,ljdd19,dv20,hsla20,hscla20}. The work of \cite{ljdd19} is most relevant to us. They consider the neural network inversion problem in generative models and prove the exact inversion problem is NP-complete.

\subsection{Data Separability and Neural Network Training}

One popular distributional assumption, in theory, is to assume the input data points to be the Gaussian distributions
\cite{zsjbd17,ly17,zsd17,glm18,bjm19,ckm20} to show the convergence of training deep neural networks. Later, convergence analysis using weaker assumptions are proposed, i.e., input data points are separable \cite{ll18}. Following \cite{ll18,als19a,als19b,all19,zpd+20}, data separability plays a crucial role in deep learning theory, especially in showing the convergence result of over-parameterized neural network training. Denote $\delta$ is the minimum gap between all pairs data points. Data separability theory says as long as the width ($m$) of neural network is at least polynomial factor of all the parameters ($m \geq \poly(n,d,1/\delta)$), i.e., $n$ is the number of data points, $d$ is the dimension of data, and $delta$ is data separability.  
Another line of work \cite{dzps19,adh+19a,adh+19b,sy19,bpsw20,lsswy20} builds on neural tangent kernel \cite{jgh18}. It requires the minimum eigenvalue ($\lambda$) of neural tangent kernel is lower bounded. Recent work \cite{os20} finds the connection between data-separabiblity $\delta$ and minimum eigenvalue $\lambda$, i.e. $\delta \geq \lambda/n^2$. 

\subsection{Distributed Deep Learning System}
Collaboration between the edge device and cloud
server achieves higher inference speed and lowers power consumption than running the task solely on the local or remote platform. Typically there are two collaborative modes. The first is collaborative training, for which training task is distributed to multiple participants \cite{konevcny2015federated,vanhaesebrouck2016decentralized,kairouz2019advances}. The other model is collaborative inference. In such a distributed system setting, the neural network can be divided into two parts. The first few layers of the network are stored in the local edge device, while the rest are offloaded to a remote cloud server. Given an input, the edge device calculates the output of the first few layers and sends it to the cloud. Then cloud perform the rest of computation and sends the final
results to each edge device  \cite{eshratifar2019jointdnn,hauswald2014hybrid,kang2017neurosurgeon,teerapittayanon2017distributed}. In our work, we focus on tackling data recovery problem under collaborative inference mode.

\subsection{Data Security}
The neural network inversion problem has been extensively investigated in recent years~\cite{fredrikson2015model,he2019model,ljdd19,zhang2020secret}. 
As used in this paper, the general approach is to cast the network inversion as an optimization problem and uses a problem specified objective. 
In particular,
\cite{fredrikson2015model} proposes to use confidence in prediction as to the optimized objective. \cite{he2019model} uses a regularized maximum likelihood estimation. 
Recent work~\cite{zhang2020secret} also proposes to use GAN to do the model inversion.

Motivated by the success of mixing data~\cite{zcdl18}, there is a line of work focusing on using data augmentation to achieve security \cite{fu2019mixup,hsla20,hscla20}. 
The most recent result \cite{hsla20} proposes the Instahide method, which makes a linear combination of training data and adds a random sign on each coordinate.
As these methods generally rely on data augmentation, they do not apply to our setting.

\section{Hardness of neural network inversion}
\label{sec:theory}

\subsection{Preliminaries}

We first provide the definitions for {\SAT}, {\ETH}, {\MAXSAT}, {\MAXESAT} and then state some fundamental results related to those definitions. For more details, we refer the reader to the textbook \cite{ab09}.
\begin{definition}[{\SAT} problem]\label{def:3SAT}
Given $n$ variables and $m$ clauses in a conjunctive normal form {\CNF} formula with the size of each clause at most $3$, the goal is to decide whether there exists an assignment to the $n$ Boolean variables to make the {\CNF} formula be satisfied.
\end{definition}

\begin{hypothesis}[Exponential Time Hypothesis ({\ETH}) \cite{ipz98}]
There is a $\delta>0$ such that the {\SAT} problem defined in Definition \ref{def:3SAT} cannot be solved in $O(2^{\delta n})$ time.
\end{hypothesis}
{\ETH} is a stronger notion than {\NP}$\neq\mathsf{P}$, and is well acceptable the computational complexity community. Over the few years, there has been work proving hardness result under {\ETH} for theoretical computer science problems \cite{cck+17,m17,cfm18,bgkm18,dm18,km18} and machine learning problems, e.g. matrix factorizations \cite{agkm12,rsw16,swz17,bbb+19}, tensor decomposition \cite{swz19}.  
There are also variations of {\ETH}, e.g. Gap-{\ETH} \cite{d16,d17,mr16} and random-{\ETH} \cite{f02,rsw16}, which are also believable in the computational complexity community. 

\begin{definition}[{\MAXSAT}]
Given $n$ variables and $m$ clauses, a conjunctive normal form {\CNF} formula with the size of each clause at most $3$, the goal is to find an assignment that satisfies the largest number of clauses.
\end{definition}

We use {\MAXESAT} to denote the version of {\MAXSAT} where each clause contains exactly $3$ literals.
\begin{theorem}[\cite{h01}]
For every $\delta > 0$, it is {\NP}-hard to distinguish a satisfiable instance of {\MAXESAT} from an instance where at most a $7/8+\delta$ fraction of the clauses can be simultaneously satisfied.
\end{theorem}

\begin{theorem}[\cite{h01,mr10}]
Assume {\ETH} holds. For every $\delta>0$, there is no $2^{o(n^{1-o(1)})}$ time algorithm to distinguish a satisfiable instance of {\MAXESAT} from an instance where at most a fraction $7/8+\delta$ of the clauses can be simultaneously satisfied.
\end{theorem}

We use {\MAXESATB} to denote the restricted special case of {\MAXSAT} where every variable occurs in at most $B$ clauses. H{\aa}stad \cite{h00} proved that the problem is approximable to within a factor $7/8+1/(64B)$ in polynomial time, and that it is hard to approximate within a factor $7/8+1/(\log B)^{\Omega(1)}$. In 2001, Trevisan improved the hardness result,
\begin{theorem}[\cite{t01}]
\label{thm:3satb}
Unless {\RP}={\NP}, there is no polynomial time $(7/8+5/\sqrt{B})$-approximate algorithm for {\MAXESATB}. 
\end{theorem}

\begin{theorem}[\cite{h01,t01,mr10}]\label{thm:hardness_max_esatb}
\label{thm:eth-3sat}
Unless {\ETH} fails, there is no $2^{o(n^{1-o(1)})}$ time $(7/8+5/\sqrt{B})$-approximate algorithm for {\MAXESATB}. 
\end{theorem}




\subsection{Our results}
We provide a hardness of approximation result for the neural network inversion problem. In particular, we prove unless {\RP}={\NP}, there is no polynomial time that can approximately recover the input of a two-layer neural network with ReLU activation function\footnote{We remark there is a polynomial time algorithm for one layer ReLU neural network recovery}. Formally, consider the inversion problem
\begin{align}
\label{eq:recovery}
h(x) = z, \quad x \in [-1, 1]^{d},
\end{align}
where $z \in \R^{m_2}$ is the hidden layer representation, $h$ is a two neural network with ReLU gates, specified as
\begin{align*}
    h(x) = W_2 \sigma(W_1 x + b), \quad W_2 \in \R^{m_2\times m_1}, W_1 \in \R^{m_1 \times d}, b\in \R^{m_1}
\end{align*}
We want to recover the input data $x \in [-1,1]^{d}$ given hidden layer representation $z$ and all parameters of the neural network (i.e., $W^{(1)}, W^{(2)}, b$). 
It is known the decision version of neural network inversion problem is $\mathsf{NP}$-hard~\cite{ljdd19}. It is an open question whether approximation version is also hard. We show a stronger result which is, it is hard to give to constant approximation factor. 
Two notions of approximation could be consider here, one we called {\em solution approximation}
\begin{definition}
[Solution approximation]
\label{def:approx1}
Given a neural network $h$ and hidden layer representation $z$, we say $x'\in [-1, 1]^{d}$ is an $\epsilon$ approximation solution for Eq.~\eqref{eq:recovery}, if there exists $x\in [-1,1] \in\R^{d}$, such that 
\begin{align*}
\| x - x' \|_2 \leq \epsilon \sqrt{d}  \text{  and  } h(x) = z.
\end{align*}
\end{definition}
Roughly speaking, solution approximation means we recovery an approximate solution. The $\sqrt{d}$ factor in the above definition is a normalization factor and it is not essential.

One can also consider a weaker notion, which we called {\em function value approximation}
\begin{definition}
[Function value approximation]
\label{def:approx2}
Given a neural network $h$ and hidden layer representation $z$, we say $x'\in [-1, 1]^{d}$ is $\epsilon$-approximate of value to Eq.~\eqref{eq:recovery}, if 
\begin{align*}
\| h(x') - y \|_2 \leq \epsilon \sqrt{m_2} .
\end{align*}
\end{definition}
Again, the $\sqrt{m_2}$ factor is only for normalization. Suppose the neural network is $G$-Lipschitz continuous for constant $G$ (which is the case in our proof), then an $\epsilon$-approximate solution implies $G\epsilon$-approximation of value. 
For the purpose of this paper, we focus on the second notion (i.e., function value approximation). 
Given our neural network is (constant)-Lipschitz continuous, this immediately implies hardness result for the first one.


Our theorem is formally stated below. In the proof, we reduce from {\MAXSATB} and utilize Theorem~\ref{thm:3satb}
\begin{theorem}
\label{thm:hardness}
There exists a constant $B > 1$, unless {\RP} = {\NP}, it is hard to $\frac{1}{60B}$-approximate Eq.~\eqref{eq:recovery} .
Furthermore, the neural network is $O(B)$-Lipschitz continuous, and therefore, it is hard to find an $\Omega(1/B^2)$ approximate solution to the neural network.
\end{theorem}

Using the above theorem, we can see that by taking a suitable constant $B > 1$, the neural network inversion problem is hard to approximate within some constant factor under both definitions. In particular, we conclude
\begin{theorem}[Formal statement of Theorem~\ref{thm:rp}]
Assume $NP \neq RP$, there exists a constant $\epsilon > 0$, such that there is no polynomial time algorithm that is able to give an $\epsilon$-approximation to neural network inversion problem.
\end{theorem}

\begin{proof}[Proof of Theorem~\ref{thm:hardness}]
Given an 3SAT instance $\phi$ with $n$ variables and $m$ clause, where each variable appears in at most $B$ clauses, we construct a two layer neural network $h_{\phi}$ and output representation $z$ satisfy the following:
\begin{itemize}
\item Completeness. If $\phi$ is satisfiable, then there exists $x\in [0,1]^{d}$ such that $h_{\phi}(x) = z$.
\item Soundness. For any $x$ such that $\|h_{\phi}(x) - z\|_{2} \leq \frac{1}{60B}\sqrt{m_2}$, we can recover an assignment to $\phi$ that satisfies at least $\left( \frac{7}{8} + \frac{5}{\sqrt{B}}  \right) m$ clauses
\item Lipschitz continuous. The neural network is $O(B)$-Lipschitz.
\end{itemize}
We set $d = n$, $m_1 = m + 200B^{2}n$ and $m_2 = m + 100B^{2}n$. For any $j \in [m]$, we use $\phi_{j}$ to denote the $j$-th clause and use $h_{1, j}(x)$ to denote the output of the $j$-th neuron in the first layer, i.e., $h_{1, j}(x) = \sigma(W^{(1)}_{j}x + b_{i})$, where $W^{(1)}_{j}$ is the $j$-th row of $W^{(1)}$. For any $i \in [n]$, we use $X_i$ to denote the $i$-th variable.

Intuitively, we use the input vector $x \in [-1,1]^{n}$ to denote the variable, and the first $m$ neurons in the first layer to denote the $m$ clauses. By taking 
\begin{align*}
W_{j, i}^{(1)} = \left\{
\begin{matrix}
1, & X_i \in \phi_j;\\
-1, & \bar{X}_{i} \in \phi_{j};\\
0, & \text{otherwise}.
\end{matrix}
\right. \quad \text{and} \quad b_{j} = -2
\end{align*}
for any $i \in [n], j\in [m]$, and viewing $x_i = 1$ as $X_i$ to be false and $x_i = -1$ as $X_i$ to be true. One can verify that $h_{1, j}(x) = 0$ if the clause is satisfied, and $h_{1, j}(x) = 1$ if the clause is unsatisfied. We simply copy the value in the second layer $h_{j}(x) = h_{1, j}(x)$ for $j \in [m]$.

For other neurons, intuitively, we make $100B^{2}$ copies for each $|x_i|$ ($i \in n$) in the output layer. This can be achieved by taking 
\begin{align*}
h_{m + (i - 1)\cdot 100 B^2 + k}(x) = h_{m + (i - 1)\cdot 100 B^2 + k}(x) + h_{1, m + 100 B^2n + (i - 1)\cdot 100 B^2 + k}(x)
\end{align*}
and set 
\begin{align*}
h_{1, m + (i - 1)\cdot 100 B^{2} + k}(x) = \max\{x_i, 0 \} \quad h_{1, m + 100 B^2 n + (i - 1) \cdot 100 B^2 + k}(x) = \max\{-x_i, 0\}
\end{align*}
for any $i \in [n], k \in [100B^{2}]$.
Finally, we set the target output as
\begin{align*}
z = (\underbrace{0, \cdots, 0}_{m}, \underbrace{1, \cdots, 1}_{100B^2n})
\end{align*}


We are left to prove the three claims we made about the neural network $h$ and the target output $z$. For the first claim, suppose $\phi$ is satisfiable and $X= (X_1, \cdots, X_n)$ is the assignment. Then as argued before, we can simply take $x_i = 1$ if $X_i$ is false and $x_i =-1$ is $X_i$ is true. One can check that $h(x) = z$.


For second claim, suppose we are given $x\in [-1, 1]^{d}$ such that 
\begin{align*}
\|h(x) - z\|_{2} \leq \frac{1}{60B}\sqrt{m_2}
\end{align*} 
We start from the simple case when $x$ is binary, i.e., $x \in \{-1, 1\}^{n}$. Again, by taking $X_i$ to be true if $x_i = -1$ and $X_i$ to be false when $x_i = 0$. One can check that the number of unsatisfied clause is at most 
\begin{align}
\label{eq:error1}
\|h(x) - z\|_2^2 
\leq & ~ \frac{1}{3600B^2}m_2 \notag \\
= & ~ \frac{1}{3600B^2}(m + 100B^2n) \notag \\
\leq & ~ \frac{1}{12}m + \frac{1}{3600B^2}m  \\
\leq & ~ \frac{1}{8}m - \frac{5}{\sqrt{B}} m\notag 
\end{align}
The third step follows from $n \leq 3m$, and the last step follows from $B \geq 15000$. 

Next, we move to the general case that $x \in [-1, 1]^{d}$. We would round $x_{i}$ to $-1$ or $+1$ based on the sign. Define $\bar{x} \in\{-1, 1\}^{n}$ as
\begin{align*}
\bar{x}_i = \arg\min_{t\in \{-1, 1\}}|t - x_i|
\end{align*}
We prove that $\bar{x}$ induces an assignment that satisfies $(\frac{7}{8} + \frac{5}{\sqrt{B}})m$ clauses. It suffices to prove
\begin{align}
\label{eq:error2}
\|h(\bar{x}) - z\|_2^{2} - \|h(x) - z\|_2^{2} \leq \frac{3}{100} m
\end{align}
since this implies the number of unsatisfied clause is bounded by
\begin{align*}
\|h(\bar{x}) - z\|_{2}^{2} 
\leq & ~ \|h(x) - z\|_2^{2} +  (\|h(\bar{x}) - z\|_2^{2} - \|h(x) - z\|_2^{2}) \\
\leq & ~ ( \frac{1}{12}m + \frac{1}{36B^2}m ) + \frac{3}{100} m \\
\leq & ~ \frac{1}{8}m - \frac{1}{5\sqrt{B}} m,
\end{align*}
where the second step follow from Eq.~\eqref{eq:error1}\eqref{eq:error2}, and the last step follows from $B \geq 10^{7}$. 

We define $\Delta_i := |\bar{x}_i - x_i| = 1 - |x_i| \in [0, 1]$ and $T : = m + 128B^2n$. Then we have
\allowdisplaybreaks
\begin{align*}
    \|h(\bar{x}) - z\|_2^2- \|h(\bar{x}) - z\|_2^2= &~ \sum_{j=1}^{T}(h_{j}(\bar{x}) - z_j) - (h_j(x) - z_j)^2\\
    = &~ \sum_{j=1}^{m}(h_{j}(\bar{x}) - z_j)^2 - (h_j(x) - z_j)^2 \\ 
    & ~ + \sum_{j=m+1}^{T}(h_{j}(\bar{x}) - z_j)^2 - (h_j(x) - z_j)^2\\
    = &~ \sum_{j=1}^{m}h_{j}(\bar{x})^2 - h_j(x)^2  - 100B^2\sum_{i=1}^{n}\Delta_i^2\\
    \leq &~ 2\sum_{j=1}^{m}|h_{1,j}(\bar{x}) -h_{1,j}(x)|  - 100B^2\sum_{i=1}^{n}\Delta_i^2\\
    \leq &~ 2\sum_{j=1}^{m}\sum_{i \in \phi_j}\Delta_i  - 100B^2\sum_{i=1}^{n}\Delta_i^2\\
    \leq &~ 2B\sum_{i=1}^{n}\Delta_i - 100B^{2}\sum_{i=1}^{n}\Delta_i^2\\
    \leq &~ \frac{n}{100}\\
    \leq&~ \frac{3m}{100}.
\end{align*}
The third step follow from $z_j = 0$ for $j \in [m]$ and for $j \in \{m +1, \cdots, m+100B^2n\}$, $z_j = 1$, $\|h_j(\bar{x}) - z_j\| = 0$ and $\|h_j(x) - z_j\|_2^{2} = \Delta_i$ given $j \in [m + (i - 1)\cdot 100B^2+ 1, i\cdot 100B^2]$. 
The fourth step follows from that $h_{j}(x) = h_{1,j}(x)\in [0,1]$ for $j \in [m]$.
The fifth step follows from the 1-Lipschitz continuity of the ReLU.
The sixth step follows from each variable appears in at most $B$ clause. 
This concludes the second claim.

For the last claim, by the Lipschitz continuity of ReLU, we have for any $x_1, x_2$
\begin{align*}
    h(x_1) - h(x_2) 
    = & ~ W^{(2)}\sigma(W^{(1)}x_1 + b) - W^{(2)}\sigma(W^{(1)}x_2 + b) \\
    \leq & ~ \|W^{(2)}\|\cdot \|W^{(1)}\|\|x_1 - x_2\|_2 
\end{align*}
It is easy to see that 
\begin{align*}
\|W^{(2)}\| \leq 2
\end{align*}
\text{ and }
\begin{align*}
    \|W^{(2)}\| \leq \sqrt{200B^2 + 3B} \leq \sqrt{203 B^2} \leq 15 B ,
\end{align*}
where the second step follows from $B \geq 1$.

Thus concluding the proof.

\end{proof}

By assuming ETH and using Theorem~\ref{thm:eth-3sat}, we can conclude
\begin{corollary}[Formal statement of Corollary~\ref{cor:eth}]
Unless ETH fails, there exists a constant $\epsilon > 0$, such that there is no $2^{o(n^{1 - o(1)})}$ time algorithm that is able to give an $\epsilon$-approximation to neural network inversion problem.
\end{corollary}
The proof is similar to Theorem~\ref{thm:hardness}, we omit it here.

\section{Details of Data Recovery Experiments}
\label{sec:recovery}
\subsection{Inversion Model Details for Synthetic Dataset}
\label{app:invert1}
In synthetic experiment, a malicious attacker recover original input data $x \in \R^d$ by solving the the following optimization:
\begin{equation*} \label{eq:attack1}
      x^* = \arg\min_{s \in \R^d} \| h(s) - z \|_1 
\end{equation*}
To estimate the optimal, we run an SGD optimizer with a learning rate of 0.01 and decayed weight $10^{-4}$ for 500 iterations. We test data recovery results on all the 200 testing samples. Namely, we solve the above optimization problems 200 times. Each time for a testing data point.

\subsection{Inversion Model Details for Benchmark Dataset}
\label{app:invert2}
In benchmark experiment, a malicious attacker recover original input data $x \in \R^d$ by solving the the following optimization:
\begin{equation*} \label{eq:attack2}
      x^* = \arg\min_{s \in \R^d} \| h(s) - z \|_2 +  \zeta \sum_{i,j} ( (s_{i+1, j} - s_{i,j})^2 + (s_{i, j+1} - s_{i,j})^2 )^{1/2},
\end{equation*}
where $i,~j$ are the indexes of pixels in an image.

To estimate the optimal, we run an SGD optimizer with a learning rate of 10 and decayed weight $10^{-4}$ for 500 iterations. We used a grid searching on the space of $\zeta$. We find that the best data recovery comes from $\zeta = 0.01$ for SVHN dataset and $\zeta = 10^{-5}$ for MNIST and FashionMNIST by grid search.

\subsection{Quantitative Metrics for Image Similarity Measurement}
\label{sec:metric}
We adopt the following two known metrics to measure the similarity between $x^*$ and $x$:
\begin{itemize}
 \vspace{-2mm}
    \item Normalized structural similarity index metric ({\bf SSIM}), a perception-based metric that considers the similarity between images in structural information, luminance and contrast. It is widely used in image and video compression research to quantify the difference between the original and compressed images.  The detailed calculation can be found in \cite{wbs+04}. We normalize SSIM to take value range $[0,1]$ (original SSIM takes value range $[-1, 1]$).
    \item  Perceptual similarity ({\bf PSIM}). Perceptual loss \cite{johnson2016perceptual} has been widely used for training image generation and style transferring models \cite{johnson2016perceptual,lucas2019generative,wang2018perceptual}. It emerges as a novel measurement for evaluating the discrepancy between high-level perceptual features that extracted by deep learning model of the reconstructed image and ground-truth image. We define PSIM as $1-$ perceptual loss.
    \vspace{-2mm}
\end{itemize}


\section{Additional Experimental Results}
\label{sec:add-exp}
\subsection{Compare Penalty Strategies}
\label{sec:unicon}
A natural approach arise to reduce data separability could be adding a penalty on the pair-wise distance for the data representations within a class. We name this approach as {\unicon}. Its loss function denoted as ${{\cal L}}_{ \mathrm{unicon} }$ can be written as:
    \begin{align*}
        {{\cal L}}_{ \mathrm{unicon} } = & ~ \frac{1}{C} \frac{1}{ |{\cal C}_c| \cdot (|{\cal C}_c|-1) } \sum_{c \in \cal C} \sum_{i \in {\cal C}_c} \sum_{j \in {\cal C}_c} \| h(x_i) - h(x_j) \|_2^2,
    \end{align*}
The final objective function ${\cal L} := {\cal L}_{\mathrm{class}} + \lambda \cdot {\cal L}_{\mathrm{unicon}}$. This approach is similar to contrastive learning \cite{khosla2020supervised}. However, we observed that the approach is not as ideal as our proposed {\mixcon}, in the sense of defending inversion attack. The intuition is that {\mixcon} can induce confusing patterns to fool the neural network learning typical patterns from a class.  Here we show the visualization for the three benchmark datasets in Figure \ref{fig:unicon}. We select $\lambda = 1$ for MNIST and FashionMNIST and  $\lambda = 0.5$ for SVHN in both {\unicon} and {\mixcon}.  Then we choose the $\beta = 1e-4$ for {\mixcon} to match the accuracy to Vanilla and {\unicon}. We use the same training and testing of {\mixcon} for {\unicon} experiment. From the representative samples (while typical to the rest of the data samples), we observe worse data recovery quality of {\mixcon}. Notably, the recovered results from {\unicon} keep the pattern of their class. While {\mixcon} results in more blurred and indistinguishable patterns across classes. We compare the quantitative evaluation results between {\mixcon} and {\unicon} in Table \ref{tab:group_similarity} \footnote{I have presented the comparison between {\mixcon} and vanilla training in Table \ref{tab:similarity}.}.  We use metric SSIM and PSIM to evaluate the similarity between the recovered image and the original image. Lower scores indicate worse data recovery results. The data recovery experiment is performed on 100 testing samples, and we report the mean $\pm$ std and worst case (the best-recovered data) results. Except for the PSIM scores evaluated on MNIST, we get conformable evidence showing {\mixcon} training is apt to defend inversion.
\begin{figure}[h]
    \centering
    \includegraphics[width=\linewidth]{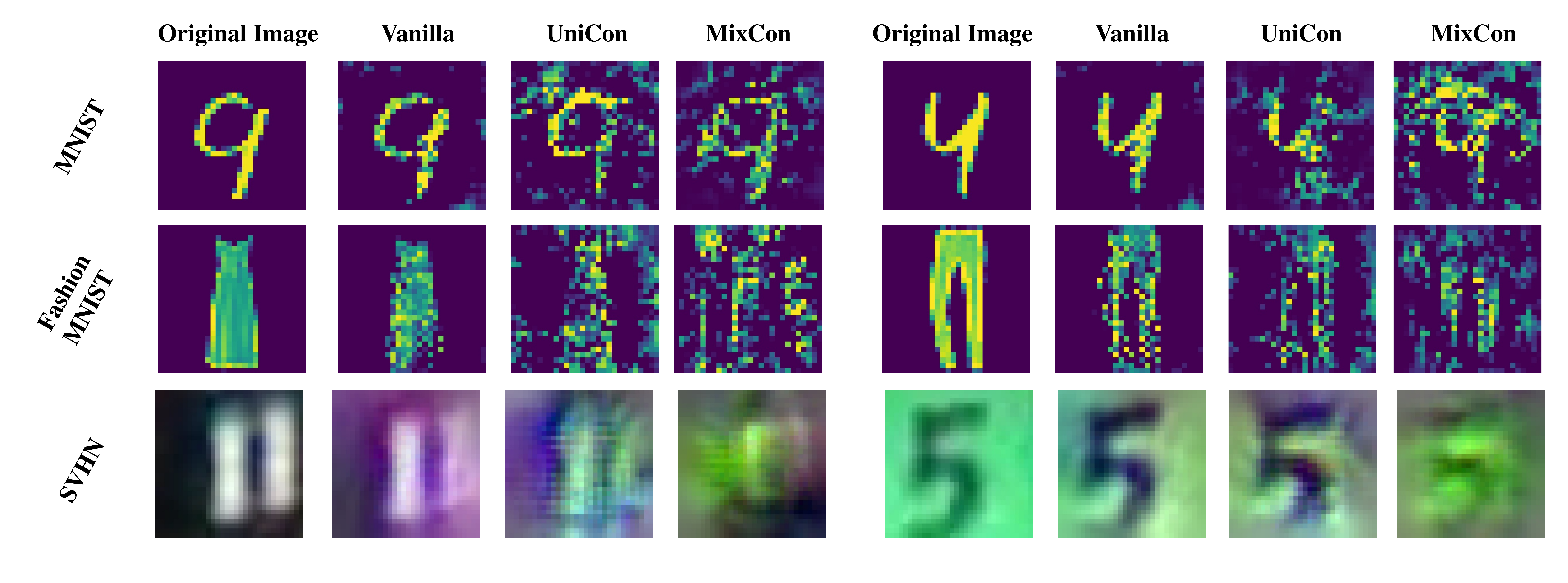}
    \caption{Qualitative evaluation for image inversion results.}
    \label{fig:unicon}
\end{figure}

\begin{table}[htbp]
\centering
\resizebox{\textwidth}{!}{%
\begin{tabular}{l|cc|cc|cc} \toprule
 & \multicolumn{2}{c|}{{\bf MNIST}} & \multicolumn{2}{c|}{{\bf FashionMNIST}} & \multicolumn{2}{c}{{\bf SVHN}} \\
 & \multicolumn{1}{c}{\begin{tabular}[c]{@{}c@{}}{\unicon}\\ $\lambda=1.0$ \end{tabular}} & \multicolumn{1}{c|}{\begin{tabular}[c]{@{}c@{}}{\mixcon}\\ $(\lambda=1.0, \beta = 10^{-4})$\end{tabular}} & \multicolumn{1}{c}{\begin{tabular}[c]{@{}c@{}}{\unicon}\\ $\lambda=1.0$ \end{tabular}} & \multicolumn{1}{c|}{\begin{tabular}[c]{@{}c@{}}{\mixcon}\\ $(\lambda=1.0, \beta = 10^{-4})$\end{tabular}} & \multicolumn{1}{c}{\begin{tabular}[c]{@{}c@{}}{\unicon}\\ $\lambda=0.5$ \end{tabular}} & \multicolumn{1}{c}{\begin{tabular}[c]{@{}c@{}} {\mixcon}\\ $(\lambda=0.5, \beta = 10^{-4})$\end{tabular}}  \\
 \hline
Acc ($\%$) & $99.2$ & $98.6$ & $89.6$ & $88.9$ & $88.3$ & $88.2$  \\
SSIM & $0.31 \pm 0.11(0.59)$  & $ 0.14 \pm 0.11 (0.48)$ & $0.19 \pm 0.09 (0.53)$ & $0.17 \pm 0.09 (0.52)$ & $0.67 \pm 0.11 (0.91)$ & $0.61 \pm 0.15 (0.84)$ \\
PSIM & $0.41 \pm 0.07 (0.60)$ & $0.44 \pm 0.07 (0.69)$ & $0.45 \pm 0.07 (0.64)$ & $0.42 \pm 0.08 (0.66)$  & $0.62 \pm 0.05 (0.75)$ & $0.59 \pm 0.07 (0.72)$\\
\bottomrule
\end{tabular}%
}
\caption{\small Quantitative evaluations for image recovery results. For fair evaluation, we match the data utility (accuracy) for Vanilla and {\mixcon}. SSIM and PSIM are measured on 100 testing samples. Those scores are presented in mean $\pm$ std and worst-case (in parentheses) format. The smaller scores indicate harder data recovery.} 
\vspace{-4mm}
\label{tab:group_similarity}
\end{table}

\subsection{Additional Results on Cifar10}
To further validate our method, we compare the inversion results on Cifar10, training with VGG16 \cite{simonyan2014very} with the default implementation. {\mixcon} is applied to the output of the second convolutional block. The classification loss is a cross-entropy loss. Neural network training uses an SGD optimizer with a momentum of 0.1 and a weight decay of 5e-4. The total training epoch is 300, and the batch size is 300. The initial learning is 0.1, which decreases to 0.01 after the 150th epoch and decreases to 0.001 after the 250th epoch. To match accuracy with vanilla training, we select $\lambda=1$ and $\beta=1e-8$ for {\mixcon}, which achieves 91.70\% accuracy, while the vanilla training achieves 92.03\% accuracy. 

After training the network, we follow the white-box inversion attack approach by learning a decoder network to invert~\cite{b95,db16}. Specifically, we have modified the existing generative adversarial neural network (GAN)~\cite{radford2015unsupervised} by adding the last term in Eq. \ref{eq:attackgan} and aim to generate an input-like image from hidden layer output as a supervised learning task.  
Suppose the target function is $h(x)$ (the trained encoder) and the attacker aims to estimate input of $x$ given its output $z = h(x)$, by solving an optimization problem for the following loss function: 
\begin{align} \label{eq:attackgan}
      L_{\text{attack}} 
      = -\frac{1}{N}\sum_{i=1}^{N} \Big( \log(D(x_i)) + \log(1-D(G(z_i))) 
    + \xi\| x_i -  G(z_i)\|_1 + \xi \| z_i -  \Phi(G(z_i))\|_1 \Big),
\end{align}
where we set $\xi=10$, $D$ denotes discriminator and $G$ denotes generator. For training the attacker, we iteratively optimize $D$ and $G$ for 50 epochs. The batch size is set as 128. We use Adam optimizers with a learning rate of 0.002. We show the visualization results in Figure \ref{fig:cifar} and quantitative evaluation in Table \ref{tab:cifar_res}. Figure \ref{fig:cifar} shows the recovered image using {\mixcon} training strategy is much harder to recognize and less perceptually similar to the original images when comparing to vanilla training. The lower scores of {\mixcon} in Table \ref{tab:cifar_res} attest our findings.
\begin{figure}[h]
    \centering
    \includegraphics[width=\linewidth]{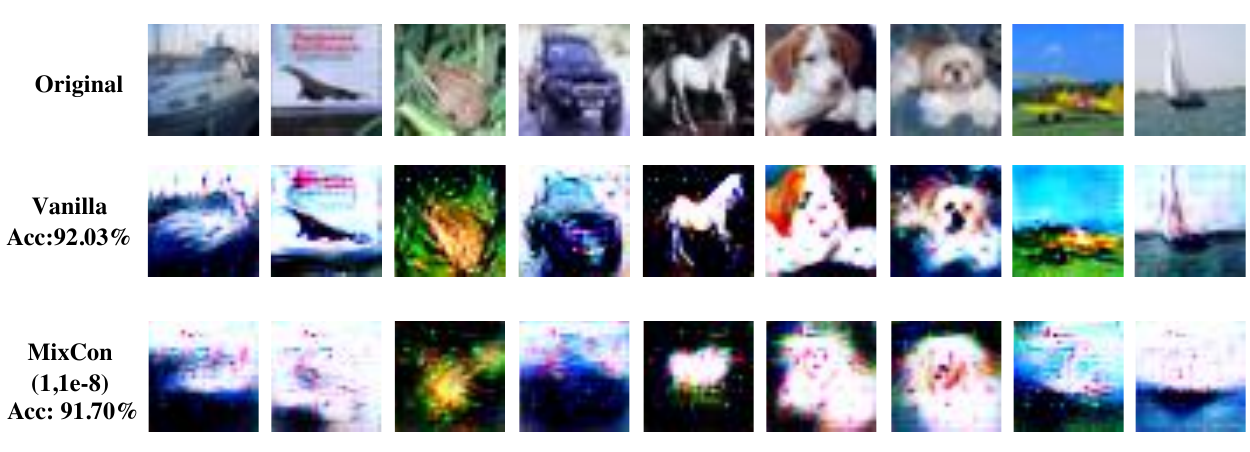}
    \caption{Qualitative evaluation for image inversion results.}
    \label{fig:cifar}
\end{figure}

\begin{table}[ht]
\centering
\begin{tabular}{l|cc|cc}
\hline
 & \multicolumn{2}{c|}{SSIM} & \multicolumn{2}{c}{PSIM} \\ \cline{2-5} 
 & mean $\pm$ std & worst & mean $\pm$ std & worst \\ \hline
\multicolumn{1}{c|}{Vanilla} & 0.53 $\pm$ 0.15 & 0.9418 & 0.92  $\pm$ 0.02 & 0.98 \\
\multicolumn{1}{c|}{{\mixcon}} & 0.33 $\pm$ 0.15 & 0.66 & 0.85 $\pm$ 0.03 & 0.94 \\ \hline
\end{tabular}%
\caption{Quantitative evaluations for image recovery results.  For fair evaluation, we match the data utility(accuracy) for Vanilla and {\mixcon}. SSIM and PSIM are measured on 100 testing samples. Lower scores indicate harder to invert.}
\label{tab:cifar_res}
\end{table}

\end{document}